%% file: elsarticle-template.tex
\documentclass{article}

\usepackage{lineno,hyperref}
\usepackage[T1]{fontenc}
\usepackage[latin9]{inputenc}
\usepackage{float}
\usepackage{textcomp}
\usepackage{amsmath}
\usepackage{amsthm}
\usepackage{amssymb}
\usepackage{graphicx}
\usepackage{tabularx,ragged2e,booktabs}
\theoremstyle{plain}
\newtheorem{theorem}{Theorem}
\newtheorem{prop}{Proposition}
\theoremstyle{plain}
\usepackage{multicol}
\usepackage{color}
\usepackage{xspace}
\usepackage{enumerate}
\usepackage{algorithm,algorithmic}
\usepackage{caption,tabularx,booktabs}
\usepackage{authblk}

\begin{document}


\title{Balanced Q-learning: Combining the Influence of Optimistic and Pessimistic
Targets}

\author{Thommen George Karimpanal}
\author{Hung Le}
\author{Majid Abdolshah}
\author{Santu Rana}
\author{Sunil Gupta}
\author{Truyen Tran}
\author{Svetha Venkatesh}
\affil{Applied Artificial Intelligence Institute, Deakin University,\\ 75 Pigdons Road, Waurn Ponds, Geelong, VIC 3216, Australia}

\maketitle


\begin{abstract}
The optimistic nature of the $Q-$learning target leads to an overestimation
bias, which is an inherent problem associated with standard $Q-$learning. Such a bias fails to account for the possibility of low returns, particularly in risky scenarios.
However, the existence of biases, whether overestimation or underestimation,
need not necessarily be undesirable. In this paper, we analytically
examine the utility of biased learning, and show that specific types
of biases may be preferable, depending on the scenario. Based on this
finding, we design a novel reinforcement learning algorithm, \emph{Balanced Q-learning},
in which the target is modified to be a convex combination of a pessimistic
and an optimistic term, whose associated weights are determined online,
analytically. We prove the convergence of this algorithm in a tabular
setting, and empirically demonstrate its superior learning performance
in various environments.
\end{abstract}


\section{Introduction\label{sec:intro}}

\input{intro.tex}

\section{Balanced Targets\label{sec:Exploiting-biases-1}}

\input{targets.tex}

\section{Desirability of Biases\label{sec:Exploiting-biases}}

\input{biases.tex}

\section{The Balanced Q-learning Framework\label{sec:Balanced-Q-learning}}

\input{method.tex}

\section{Experiments\label{sec:Experiments}}

\input{exp.tex}

\section{Related Work\label{sec:bg-1}}

\input{bg.tex}

\section{Conclusion\label{sec:conclusion}}

\input{conclude.tex}

\bibliography{test}

\section*{\newpage{}\protect\pagebreak{}}

\appendix

\section*{Appendix\label{sec:supp}}

\input{supp.tex}

\end{document}

%% file: intro.tex
$Q-$learning \cite{watkins1989learningfrom} is one of the most
popular algorithms used in reinforcement learning \cite{sutton1998reinforcement}.
The update rule is simple and intuitive, wherein at each step, action
values are updated towards a target, comprised of the reward received
during the current step and a discounted estimate of the \emph{maximum}
return achievable from the next step onwards. This dependency on the
estimated maximum returns makes the $Q-$learning target overly optimistic
with regards to the expected future rewards. As a consequence, the resulting policy fails to account for potentially risky actions that could result in low rewards. This is also closely tied
to the problem of overestimation \cite{thrun1993issues}, which can
cause the policy to significantly deviate from the optimal one \cite{hasselt2010double,van2016deep}
during learning.

In recent years, several attempts \cite{hasselt2010double,van2016deep,lan2019maxmin,anschel2017averaged}
have been made to control the extent of this overestimation bias.
Some of these solutions, while eliminating the overestimation bias,
introduce underestimation, as is the case in double $Q-$learning
\cite{hasselt2010double} and double DQN \cite{van2016deep}. It
has previously been suggested by Lan et al. \cite{lan2019maxmin}
that although both types of biases are generally undesirable, they
are not strictly detrimental, and in certain scenarios, they can aid
learning. Specifically, in areas associated with high stochasticity,
overestimation is beneficial if those regions are associated with
a high value (low risk regions), as this encourages exploration into these regions. Conversely,
in highly stochastic low value regions (high risk regions), underestimation discourages
exploration into these regions, which can be beneficial for learning. This highlights
the need for a mechanism that can automatically control the degree
of overestimation/underestimation online during learning.


We draw from this intuition and develop an approach to automatically
control the agent's tendency to optimistically or pessimistically
estimate the learning target, thereby also controlling the level of
overestimation, through a novel algorithm, \emph{Balanced $Q-$learning}.
We consider the current action value estimate, and the difference
between the learning target and optimal value, and derive a temporal
relation between these quantities as the learning progresses. This
relation analytically confirms the aforementioned intuitions of Lan
et al. \cite{lan2019maxmin}, and provides an idea of the type of
scenarios that are likely to benefit from overestimation and underestimation.
The derived relation also forms the basis for the design of the balanced
$Q-$learning algorithm, in which we replace the maximization term
of the standard $Q-$learning target with a convex combination of
an optimistic maximization term and a pessimistic minimization term,
controlling the extent of influence of each term using a balancing
factor $\beta$. 
We derive an online update rule for $\beta$, and prove the convergence
of the resulting algorithm in tabular settings. We then compare the
empirical performance of balanced $Q-$learning with competing approaches
in a number of benchmark environments. In summary, our main contributions
are: 
\begin{itemize}
\item A risk-aware framework, balanced $Q-$learning to adaptively balance the extent of optimism and pessimism in each step, thereby accelerating learning. 
\item Derivation of a temporal relation between the action value estimate
and the target difference during learning. 
\item An empirical comparison of the performance of balanced $Q-$learning
with existing approaches in a variety of benchmark environments. 
\end{itemize}

%% file: targets.tex
We consider an MDP (Markov Decision Process) \cite{Puterman:1994:MDP:528623}
setting, $(\mathcal{S,A,T,R})$, where $\mathcal{S}$ is the state-space,
$\mathcal{A}$ is the action-space, $\mathcal{T}$ represents the
transition probabilities, and $\mathcal{R}$ is the reward function.
For a state $s\in\mathcal{S}$ and action $a\in\mathcal{A}$, the
$(n+1)^{th}$ estimate of the action value $Q_{n+1}(s,a)$ is given
by the temporal difference (TD) update rule:
\begin{equation}
Q_{n+1}(s,a)\leftarrow Q_{n}(s,a)+\alpha\left[Q_{T_{n}}-Q_{n}(s,a)\right],\label{eq:gen_TD}
\end{equation}
where $Q_{T_{n}}$ is the target for the TD update. 

In standard $Q-$learning and DQN, the learning target $Q_{T_{n}}$
is given by: $Q_{T_{n}}=r(s,a)+\gamma\underset{a'}{max}\thinspace Q_{n}(s',a')$, where $\gamma$ is the discount factor.
This implies that the target is dependent on the reward $r(s,a)$,
and on an estimate of the sum of the future rewards $\underset{a'}{max\thinspace}Q_{n}(s',a')$
that would be obtained from state $s'$. The \emph{max} operator encodes
the optimistic nature of $Q-$learning, which only considers the best
case scenario, where the maximum return would be obtained. This operator
also introduces an overestimation bias in the presence of stochastic
transition and/or reward functions \cite{hasselt2010double}, particularly
with the use of function approximators, as they introduce approximation
errors, which make it more likely for the \emph{max} operator to incorrectly
overvalue certain actions. Such an overvaluing of actions could lead to low returns, especially if the agent operates in adverse environments where several actions are considered to be undesirable/associated with low or negative rewards.

Following conventions by Thrun and Schwartz \cite{thrun1993issues},
in general, if $Q_{est}$ and $Q_{true}$ denote the estimated and
true $Q-$values, the \emph{max} operator leads to an estimation bias
$Z$: 
\[
Z=\gamma(\underset{a'}{max\thinspace}Q_{est}(s',a')-\underset{a'}{max\thinspace}Q_{true}(s',a'))
\]

The key intuition is that due to the \emph{max} operator, the expected
bias $\mathbb{E}(Z)$ is positive \cite{thrun1993issues}, which
implies an overestimated target. On the other hand, if we replace
the \emph{max} operator with the other extreme, the TD target would
be $Q_{T_{n}}=r(s,a)+\gamma\underset{a'}{\thinspace min}Q_{n}(s',a')$.
Such a target is based on a pessimistic estimate of the future sum
of rewards. The corresponding estimation bias can be shown to be:
\[
Z=\gamma(\underset{a'}{min\thinspace}Q_{est}(s',a')-\underset{a'}{min\thinspace}Q_{true}(s',a'))
\]

Here, the expected bias $\mathbb{E}(Z)$ would be negative. The intuition
behind balanced $Q-$learning is to appropriately weight these optimistic
and pessimistic targets to contextually promote the right type of
biases during learning. A simple linear combination of these targets
would result in the balanced TD target:
\begin{equation}
Q_{T_{n}}=r(s,a)+\gamma\left[\beta(s,a)\underset{a'}{max\thinspace}Q_{n}(s',a')+(1-\beta(s,a))\underset{a'}{min\thinspace}Q_{n}(s',a')\right]\label{eq:balancedQ}
\end{equation}
where $\beta(s,a)$\footnote{$\beta(s,a)$ is henceforth denoted as $\beta$ for brevity.}
($0\leq\beta(s,a)\leq1$) is the balancing factor, whose value is
determined online. The estimation bias for such a formulation would
be:
\[
Z=\gamma\beta(\underset{a'}{max\thinspace}Q_{est}(s',a')-\underset{a'}{max\thinspace}Q_{true}(s',a'))\]\[+\gamma(1-\beta)(\underset{a'}{min\thinspace}Q_{est}(s',a')-\underset{a'}{min\thinspace}Q_{true}(s',a'))
\]

As per the above equation, $\mathbb{E}(Z)$ can vary across the full
range by varying $\beta$ from $1$ (maximum overestimation) to $0$
(maximum underestimation). While our argument here was based on the
presence of function approximation noise, the approach can also be
used to alleviate overestimation in stochastic tabular environments,
owing to the fact that allowing $\beta<1$ during learning can diminish
the extent of overestimation that would otherwise be caused by the
\emph{max} operator.

%% file: biases.tex
As per Equation \ref{eq:gen_TD}, in order for the action value to
move towards the optimum value $Q^{*}(s,a)$, $Q_{T_{n}}$ must be
approximated to be as close to $Q^{*}(s,a)$ as possible. We refer
to $Q_{T_{n}}-Q^{*}(s,a)$ as the \textit{target difference} $t_{n}(s,a)$:
\begin{equation}
t_{n}(s,a)=Q_{T_{n}}-Q^{*}(s,a)\label{eq:target_diff}
\end{equation}

As $Q^{*}(s,a)$ is unknown, it is not possible to accurately determine
$t_{n}(s,a)$. However, as explained in the subsequent paragraphs,
a temporal relation between the target difference and the difference
between the optimal and estimated value functions can inform the nature
of biases (overestimation or underestimation) that are preferred in
different scenarios. This temporal relation is shown in Equation \ref{eq:biasperformance}
of Theorem \ref{thm:theorem1}. 
\begin{theorem}
In a finite MDP $\mathcal{M={S,A,T,R}}$, for a given
state-action pair $(s,a)$, the difference between the optimal $Q$-function
$Q^{*}(s,a)$ and the action value estimate $Q_{n+m}(s,a)$ after
$m$ updates is given by\label{thm:theorem1}:

\begin{equation}
Q^{*}(s,a)-Q_{n+m}(s,a)=(1-\alpha)^{m}\left[Q^{*}(s,a)-Q_{n}(s,a)\right]-\alpha\sum_{i=1}^{m}(1-\alpha)^{i-1}t_{n+m-i}(s,a)\label{eq:biasperformance}
\end{equation}
\label{sec:Theory}where $Q_{n}(s,a)$ is the estimate of the value
function at the $n^{th}$ update, $Q^{*}(s,a)$ is the optimal value
function, $\alpha$ is the step size, and $t(s,a)$ is the target
difference. 
\end{theorem}

\begin{proof}
(Proof of Theorem \ref{thm:theorem1} is provided in Appendix \ref{subsec:Proof-of-Theorem}) 
\end{proof}
\emph{High-value (low-risk) regions:} In Equation \ref{eq:biasperformance},
if $Q_{n}(s,a)$ is a random initial (i.e., at $n=1$) estimate of
the $Q-$function, in a stochastic high-value region of the state-action
space, it is likely that the term $Q^{*}(s,a)-Q_{n}(s,a)$ is positive.
Hence, in order for the left hand side of Equation \ref{eq:biasperformance}
to be close to $0$, it is preferable for the term $\alpha\sum_{i=1}^{m}(1-\alpha)^{i-1}t_{n+m-i}(s,a)$,
to be positive. From Equation \ref{eq:target_diff}, it is clear that
this can be achieved by ensuring the learning target generally exceeds
$Q^{*}(s,a)$, particularly with respect to the most recent updates.
Such a scenario can possibly be induced through overestimation of
the target, specifically, by using higher values of $\beta$ in Equation
\ref{eq:balancedQ}.

\emph{Low-value (high-risk) regions:} Following a similar rationale, in stochastic
low-value regions, given a randomly initialized estimate $Q_{n}(s,a)$,
the term $Q^{*}(s,a)-Q_{n}(s,a)$ is likely to be negative. Hence,
the term on the RHS $\alpha\sum_{i=1}^{m}(1-\alpha)^{i-1}t_{n+m-i}(s,a)$ must
be negative in order for $Q_{n+m}(s,a)$ to be as close as possible
to $Q^{*}(s,a)$. The weighted sum of target differences can assume
negative values if the target difference is generally negative, particularly
with respect to the most recent updates. This can be achieved by underestimating
the learning target by using lower values of $\beta$ in Equation
\ref{eq:balancedQ}.

This analytically confirms the intuitions of Lan et al. \cite{lan2019maxmin},
that underestimation is likely to be beneficial in low-value (high risk) regions
of the state-action space, and overestimation is likely to be beneficial
in high-value (low risk) regions.

%% file: method.tex
In this section, we develop the overall framework for balanced Q-learning. 

From the discussion following Theorem \ref{thm:theorem1}, we know
that based on the term $Q^{*}(s,a)-Q_{n}(s,a)$ in Equation \ref{eq:biasperformance}
being positive or negative, overestimation or underestimation biases
would be preferred respectively. If $Q^{*}(s,a)$ were known, it would
be possible to artificially modify the target $Q_{T_{n+1}}$ as:
\begin{equation}
Q_{T_{n+1}}'=Q_{T_{n+1}}+\eta[Q^{*}(s,a)-Q_{n}(s,a)]\label{eq:newtarget}
\end{equation}
where $Q_{T_{n+1}}'$ is the modified target and $\eta>0$ is a step
size hyperparameter. This way, a positive value of $Q^{*}(s,a)-Q_{n}(s,a)$
would cause the modified target $Q_{T_{n+1}}'$ to exceed $Q_{T_{n+1}}$,
thereby injecting positive biases into the system. Similarly, in response
to a negative value of $Q^{*}(s,a)-Q_{n}(s,a)$, the target would
be modified to inject negative biases into the system. As $Q^{*}(s,a)$
is unknown, at each step, we use the balanced target in its place,
adaptively injecting the appropriate level of bias by controlling
the value of the balancing factor online during learning. The corresponding
adaptive update rule for updating the balancing factor is given by
Proposition \ref{thm:betarulethm}.

\begin{figure}
\centering{}\includegraphics[width=0.9\columnwidth]{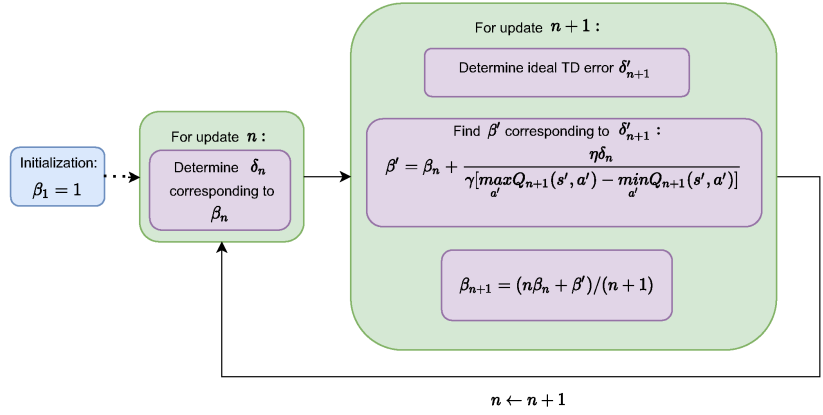}\caption{Block diagram describing the processes involved in Balanced $Q-$learning}
\label{fig:flow} 
\end{figure}

\begin{prop}
\label{thm:betarulethm}For a transition $(s,a,r,s')$ in an MDP $\mathcal{{M}}$,
where the $n^{th}$ update of the action value $Q_{n}(s,a)$ conducted
as per Equation \ref{eq:gen_TD} is associated with a TD error $\delta_{n}$,
and the modified target for the $(n+1)^{th}$ update is computed as
per Equation \ref{eq:newtarget}, with $Q^{*}(s,a)$ being approximated
by the balanced target in Equation \ref{eq:balancedQ}, the equivalent
balancing factor $\beta'$ for the $(n+1)^{th}$ update is given by
\begin{equation}
\beta'=\beta_{n}+\frac{\eta\delta_{n}}{\gamma[\underset{a'}{max\thinspace}Q_{n+1}(s',a')-\underset{a'}{min\thinspace}Q_{n+1}(s',a')]}\label{eq:betavalue-1}
\end{equation}

where $\eta$ is a step-size hyperparameter, $\gamma$ is the discount
factor and $Q_{n+1}$ is the action value estimate corresponding to
the $(n+1)^{th}$ update. 
\end{prop}

\begin{proof}
$Q^{*}(s,a)$ in Equation \ref{eq:newtarget} is initially approximated
as $Q_{T_{n}}=r(s,a)+\gamma[\beta_{n}\underset{a'}{max}Q_{n}(s',a')+(1-\beta_{n})\underset{a'}{min}Q_{n}(s',a')]$
using the balancing factor $\beta_{n}$ (Initialized to $1$). Subsequently,
Equation \ref{eq:newtarget} can be re-written as:

\begin{equation}
Q_{T_{n+1}}'=Q_{T_{n+1}}+\eta[Q_{T_{n}}-Q_{n}(s,a)]\label{eq:Qtarg}
\end{equation}

\begin{equation}
=Q_{T_{n+1}}+\eta\delta_{n}\label{eq:modtarget}
\end{equation}

where $\delta_{n}$, the TD error associated with the $n^{th}$ update
is:

\begin{equation}
\delta_{n}=r(s,a)+\gamma[\beta_{n}\underset{a'}{max}Q_{n}(s',a')+(1-\beta_{n})\underset{a'}{min}Q_{n}(s',a')]-Q_{n}(s,a)\label{eq:delta_n_beta_0}
\end{equation}

or 
\begin{equation}
\delta_{n}=Q_{T_{n}}-Q_{n}(s,a)\label{eq:TDgeneral}
\end{equation}

Corresponding to these estimates, the ideal TD error for the subsequent
(i.e., $(n+1)^{th}$) update of the state-action pair can be computed
analogous to Equation \ref{eq:TDgeneral} as: 

\begin{equation}
\delta'_{n+1}=Q'_{T_{n+1}}-Q_{n+1}(s,a)\label{eq:deltanew}
\end{equation}

Using Equation \ref{eq:modtarget}, the above relation can be expressed
as:

\begin{equation}
\delta'_{n+1}=Q_{T_{n+1}}+\eta\delta_{n}-Q_{n+1}(s,a)\label{eq:deltadash}
\end{equation}

In addition, with the current estimate of the balancing factor $\beta_{n}$,
$Q_{T_{n+1}}$ in the above equation can be further expanded as:

\[
Q_{T_{n+1}}=r(s,a)+\gamma[\beta{}_{n}maxQ_{n+1}(s',a')+(1-\beta_{n})\underset{a'}{min}Q_{n+1}(s',a')]
\]

Substituting this in Equation \ref{eq:deltadash}, we obtain the ideal
TD error associated with the $(n+1)^{th}$ update, as per the current
value of balancing factor $\beta_{n}$:

\begin{equation}
\delta'_{n+1}=r(s,a)+\gamma[\beta_{n}\underset{a'}{max}Q_{n+1}(s',a')+(1-\beta_{n})\underset{a'}{min}Q_{n+1}(s',a')]+\eta\delta_{n}-Q_{n+1}(s,a)\label{eq:firsteq}
\end{equation}

This TD error $\delta'_{n+1}$ can be assumed to be associated with
an equivalent balancing factor $\beta',$ such that:

\begin{equation}
\delta'_{n+1}=r(s,a)+\gamma\left[\beta'\underset{a'}{max\thinspace}Q_{n+1}(s',a')+\left[1-\beta'\right]\underset{a'}{min\thinspace}Q_{n+1}(s',a')\right]-Q_{n+1}(s,a)\label{eq:delta_n_1_updated}
\end{equation}

From Equations \ref{eq:firsteq} and \ref{eq:delta_n_1_updated},
we get:
\begin{equation}
\beta'=\beta_{n}+\frac{\eta\delta_{n}}{\gamma[\underset{a'}{max\thinspace}Q_{n+1}(s',a')-\underset{a'}{min\thinspace}Q_{n+1}(s',a')]}\label{eq:betavalue}
\end{equation}
\end{proof}
Hence, at each step, we use an initial estimate $\beta_{n}$, and
compute $\beta'$, (clipped to $[0,1]$ if needed; clipping ensures
that only realistic targets are used), which is used to compute the
balanced target, based on which the action values are updated. The subsequent value of $\beta_{n}$
is also updated online, as explained later.

\textbf{Interpreting the $\beta'$ update:} Equation \ref{eq:betavalue}
suggests that when $\delta_{n}=0$, an appropriate value of $\beta_{n}$
is being used to determine the balanced target, and thus, $\beta'$
assumes the same value as $\beta_{n}$. When $\delta_{n}\neq0$, the
direction of the update depends solely on $\delta_{n}$, as the denominator
$[\underset{a'}{max\thinspace}Q_{n+1}(s',a')-\underset{a'}{min\thinspace}Q_{n+1}(s',a')]$
is always positive. $\delta_{n}$, (which can be expressed as $\delta_{n}=Q_{T_{n}}-Q_{n}(s,a)$),
when positive, indicates that the estimated action value $Q_{n}(s,a)$
still falls short of the determined balanced target $Q_{T_{n}}$ (which
is determined using $\beta_{n}$). Hence, in order to drive up the
value of $Q_{n}(s,a)$, $\beta'$ is updated to a value larger than
$\beta_{n}$, increasing the subsequent update's reliance on the maximization
term. Similarly, when $\delta_{n}$ is negative, it implies that $Q_{n}(s,a)$
exceeds the balanced target $Q_{T_{n}}$. In order to drive down the
value of $Q_{n}(s,a)$, $\beta'$ is updated to a lower value, increasing
its reliance on the minimization term. Hence, $\beta'$ makes \emph{sample-specific adjustments}
to the subsequent target in order to drive up or drive down the current
action value estimate as required.


\textbf{Updating $\beta_n$:} Initialized as $1$, $\beta_{n}$ is
subsequently updated as the incremental average of $\beta'$ following each action value update:


\begin{equation}
\label{eqn:betan_update}
\beta_{n+1}=(n\beta_{n}+\beta')/(n+1)
\end{equation}

As $\beta_{n}$ controls the general extent of reliance on the best
and worst possible estimated future returns for all interactions,
it can be interpreted as an estimate of the \emph{average degree of optimism}
for a given environment. This is in contrast to $\beta'$, which is
a sample-specific degree of optimism. Algorithm \ref{alg:Balanced-Q-learning}
summarizes the steps involved in balanced DQN, a DQN variant of
balanced $Q-$learning. A tabular version of the algorithm is shown
in Appendix \ref{sec:Tabular-Implementation:}. The process is also
pictorially depicted in Figure \ref{fig:flow}.

\textbf{Practical considerations:} When
action values are learned using a DQN-like approach, it is common
to update the neural network parameters by sampling batches of transitions (As
shown in Algorithm \ref{alg:Balanced-Q-learning}). That is, corresponding each batch, a single action value update is carried out. Since the average value of $\beta_n$ (Equation \ref{eqn:betan_update}) is computed as the average value of $\beta'$ used over the action-value updates, for each batch, we compute $\beta_{batch}'$ as the mean value of $\beta'$ corresponding to samples in the batch (and treat this as a representative value for that batch), and use this value to update $\beta_n$. Doing so allows $\beta_n$ to be computed as the average value of $\beta'$ over action-value updates. We treat this as a caveat associated with the DQN-variant of our algorithm.

\begin{algorithm}[H]
\caption{Balanced DQN\label{alg:Balanced-Q-learning}}

\begin{algorithmic}[1]

\STATE \textbf{Input: }

\STATE Step sizes \textbf{$\alpha$}, $\eta,$ exploration parameter
$\epsilon$, discount factor $\gamma$, maximum number of steps $N_{max}$, Batch size $b$

\STATE Initialize replay buffer $D$, count $n=1$, $\beta_{1}=1$
and $Q-$network $Q(s,a,\theta)$

\STATE Initialize $Q_{n}$ and $Q_{n+1}$ as $Q(s,a,\theta)$

\STATE Get initial state $s$

\STATE \textbf{Output: }Learned value function $Q_{N_{max}}$\textbf{ }

\WHILE {$n\leq N_{max}$} 

\STATE Use $\epsilon-$greedy strategy to choose action $a$; observe
$r,s'$

\STATE Store transition $(s,a,r,s')$ in $D$

\STATE Sample mini-batch $B$ of size $b$ from $D$

\FOR {$m\in B$}

\STATE Compute TD error $\delta_{n}$ using $Q_{n}$ (Equation \ref{eq:delta_n_beta_0})$:$

\STATE Obtain $\beta'_{m}$ for sample $m$ as per Equation \ref{eq:betavalue}:

$\beta_{m}'=\beta_{n}+\frac{\eta\delta_{n}}{\gamma[\underset{a'}{max\thinspace}Q_{n+1}(s_{m}',a')-\underset{a'}{min\thinspace}Q_{n+1}(s_{m}',a')]}$

\STATE Clip $\beta_{m}'$ to the range $[0,1]$

\STATE Compute target 
\[
Q_{T_{m}}=r_{m}(s_{m},a_{m})+\gamma\beta_{m}'\underset{a'}{max\thinspace}Q_{n+1}(s_{m}',a')+(1-\beta_{m}')\underset{a'}{min\thinspace}Q_{n+1}(s'_{m},a')]
\]

\ENDFOR

\STATE $\beta_{batch}'=\frac{1}{b}\underset{i}\sum\beta_i'$

\STATE Update $\beta_{n+1}$ as $\beta_{n+1}=(n\beta_{n}+\beta_{batch}')/(n+1)$

\STATE Store $Q_{n}$ as the current estimate: $Q_{n}\leftarrow Q(s,a;\theta)$

\STATE Update network parameter $\theta$ using the computed targets
$Q_{T}$:

$\theta\thickapprox argmin_{\theta}[Q_{T}-Q(s,a;\theta)]^{2}$

\STATE Store $Q_{n+1}$ as the updated estimate: $Q_{n+1}\leftarrow Q(s,a;\theta)$

\STATE Update state: $s\leftarrow s'$

\STATE Update count: $n\leftarrow n+1$

\ENDWHILE

\end{algorithmic} 
\end{algorithm}

Balanced $Q-$learning can also be shown to converge in tabular environments
as long as $\eta\leq\gamma$, in addition to other standard conditions
on the step size $\alpha$ (Theorem \ref{thm:convergence}). 
\begin{theorem}
In a finite MDP ($\mathcal{S,A,T,R}),$ balanced $Q-$learning, g\label{thm:convergence}iven
by the update rule:
\[
Q_{n+1}(s,a)=Q_{n}(s,a)+\alpha_{n}(s,a)\left[Q_{T}-Q_{n}(s,a)\right]
\]
where
\[
Q_{T}=r(s,a)+\gamma\left[\beta'\underset{a'}{max\thinspace}Q_{n}(s',a')+(1-\beta')\underset{a'}{min\thinspace}Q(s',a')\right]
\]
converges to a fixed point with probability 1 as long as $\eta\leq\gamma$,
$\underset{n=1}{\overset{\infty}{\sum}}\alpha_{n}(s,a)=\infty$ and
$\underset{n=1}{\overset{\infty}{\sum}}\alpha_{n}^{2}(s,a)<\infty$
$\forall(s,a)\in\mathcal{S\times\mathcal{A}}$, where $\eta$ is a
step size hyperparameter. 
\end{theorem}

\begin{proof}
(Proof of Theorem \ref{thm:convergence} is provided in Appendix \ref{subsec:Proof-of-Theorem_convergence}) 
\end{proof}

We note that due to the `min' operator in the learning target, balanced $Q-$learning converges to a more risk-averse fixed point compared to algorithms such as standard $Q-$learning, where the learning target depends solely on a maximization term.  Although this may not be optimal in terms of reward maximization, it aids online learning performance by accounting for high-risk actions in an environment.

%% file: exp.tex
\begin{figure}
\centering{}\includegraphics[width=0.6\columnwidth]{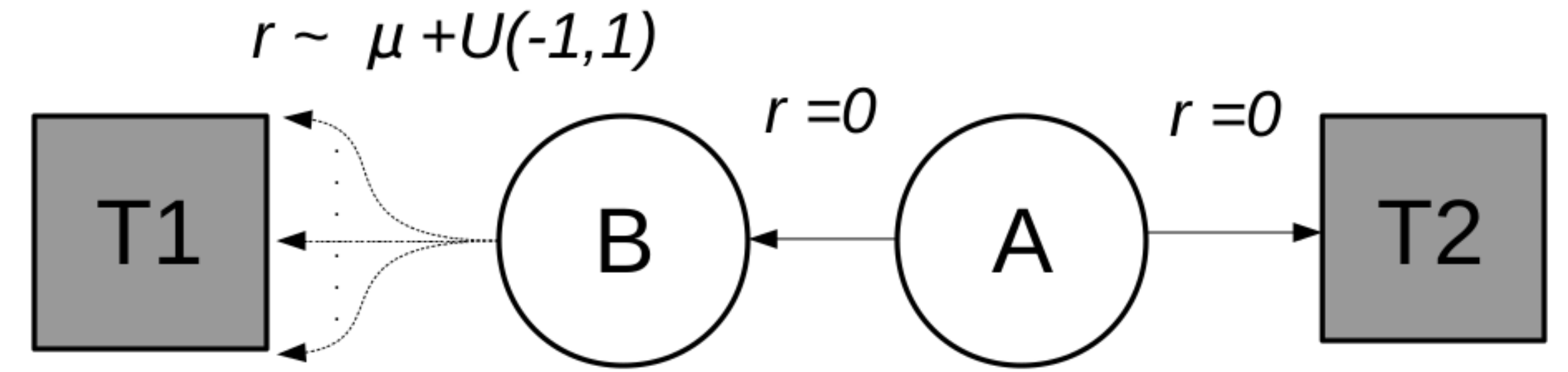}\caption{A simple MDP with two non-terminal states $A$ and $B$, and two terminal
states $T1$ and $T2$. The transition into $T1$ is associated with
a reward $r\sim\mu+U(-1,1)$, where $\mu$ is the mean reward. All
other transitions are associated with a reward of $0$. Depending
on whether $\mu$ is positive or negative, the optimal action from
\textbf{$A$ }is to move either left or right.}
\label{fig:lineworld} 
\end{figure}

In this section, we demonstrate balanced $Q-$learning on a simple
MDP (Figure \ref{fig:lineworld}), following which we present further
empirical comparisons in other environments. We choose DQN\cite{mnih2015human},
double DQN\cite{van2016deep}, maxmin DQN \cite{lan2019maxmin},
averaged DQN \cite{anschel2017averaged} and REDQ learning \cite{chen2021randomized}
as baselines for comparison.

\begin{figure}
\begin{raggedright}
\begin{minipage}[t]{0.5\textwidth}%
\noindent \begin{center}
\includegraphics[width=1\textwidth]{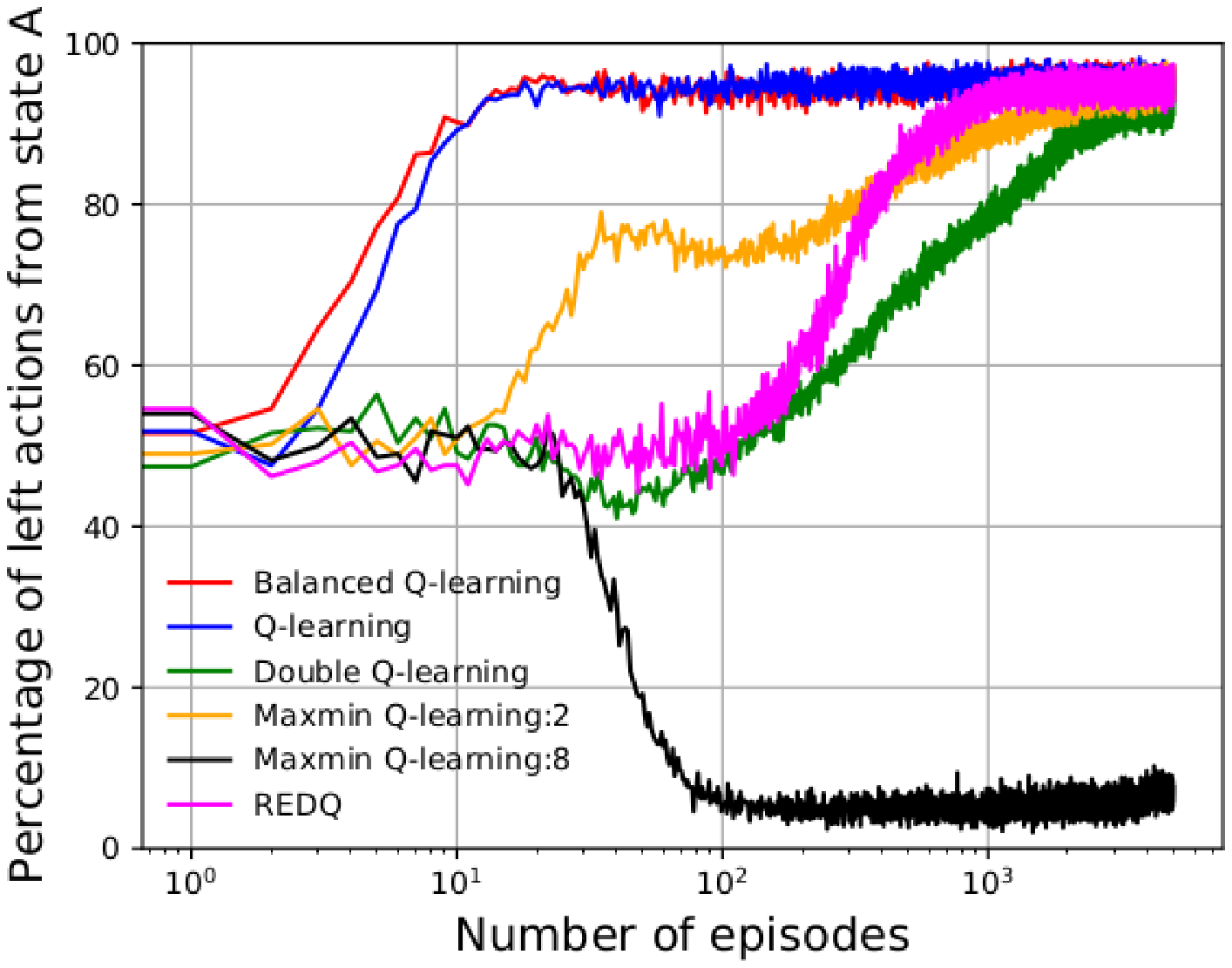} 
\par\end{center}
\begin{center}
(a) Mean reward $\mu=+0.1$. Higher values are better. 
\par\end{center}%
\end{minipage}%
\begin{minipage}[t]{0.5\textwidth}%
\noindent \begin{center}
\includegraphics[width=1\textwidth]{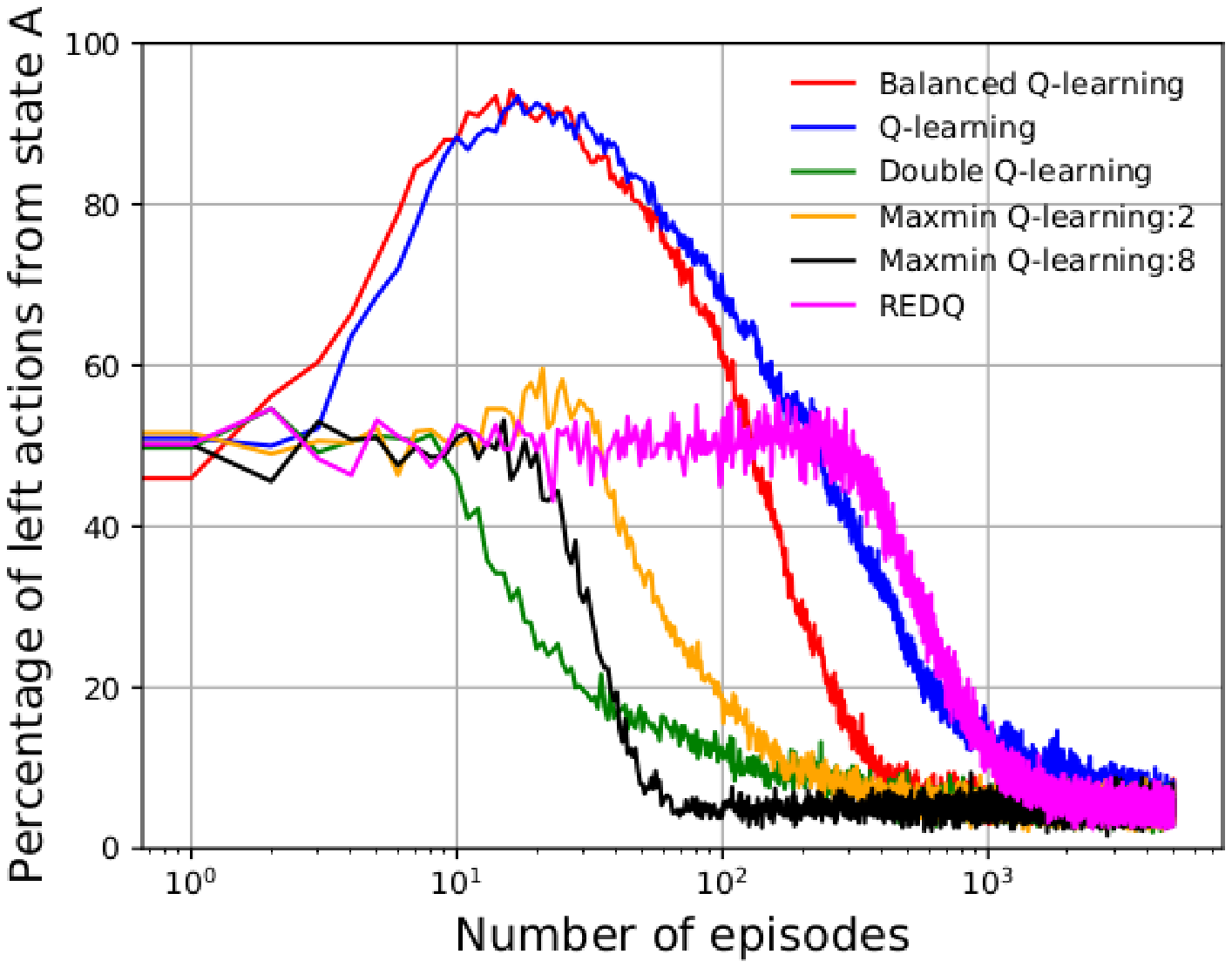} 
\par\end{center}
\begin{center}
(b) Mean reward $\mu=-0.1$. Lower values are better. 
\par\end{center}%
\end{minipage}%
\par\end{raggedright}
\raggedright{}\caption{\label{fig:lineworld_leftactions}The percentage of left actions taken
from state $A$ during training, computed over $500$ trials, for
the two reward settings: (a) $\mu=+0.1$ and (b) $\mu=-0.1$.}
\end{figure}

\begin{figure}
\centering
\includegraphics[width=0.5\textwidth]{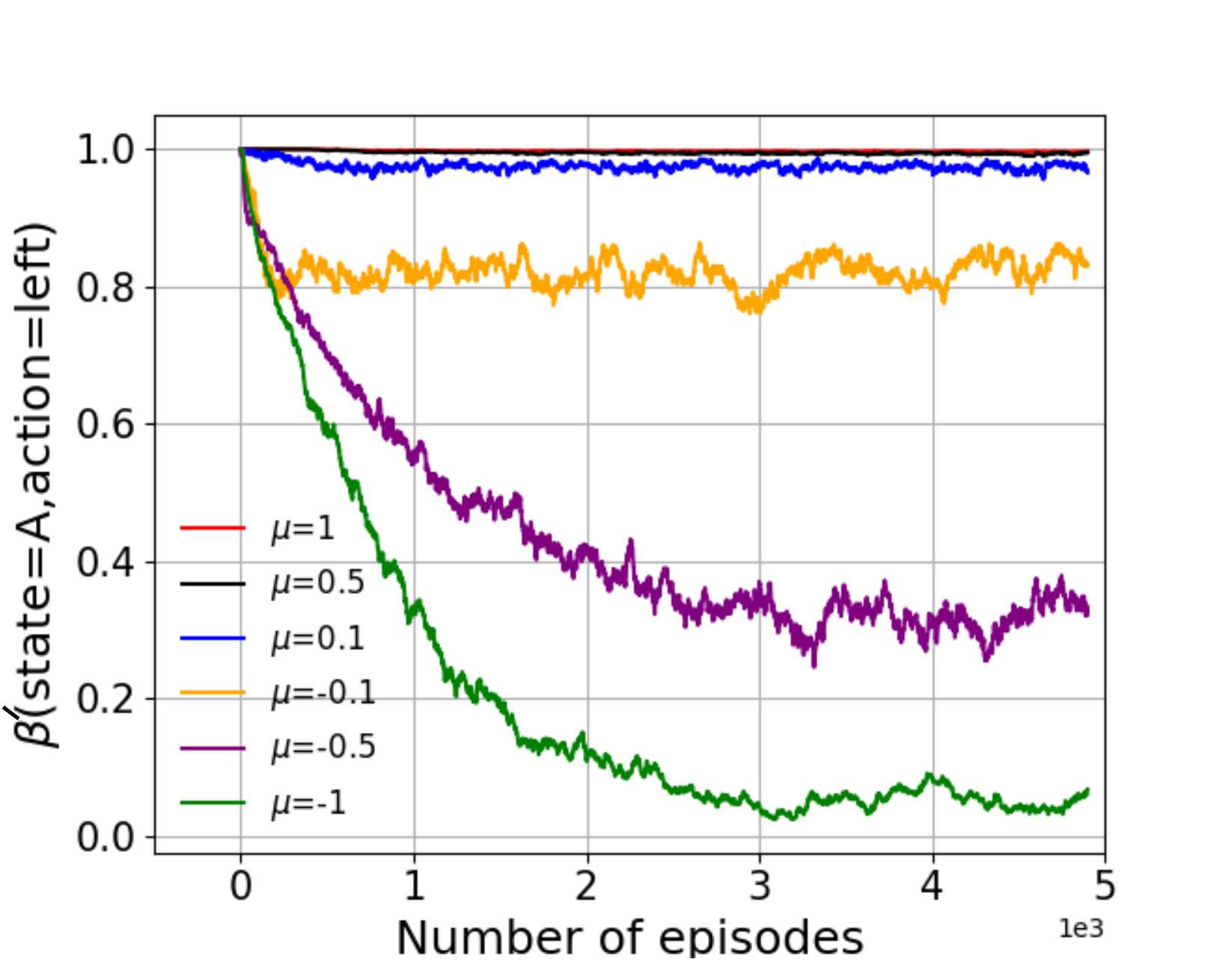} 
\caption{\label{fig:lineworld_leftactions_betavar}Plot showing the variation of $\beta'$ for left actions taken from state
A of Figure \ref{fig:lineworld}, for different values of mean rewards
$\mu$..}
\end{figure}

\textbf{Robustness to different bias preference scenarios: }Certain
learning scenarios could benefit from an inherent overestimation bias,
whereas in other scenarios, underestimation may be preferred. Hence,
it is undesirable for the underlying learning algorithm to be strictly
associated with either of these bias types. The goal of balanced $Q-$learning
is to adaptively provide the correct type of biases, depending on
the scenario.

We demonstrate this property using a simple MDP shown in Figure \ref{fig:lineworld},
previously used to study overestimation issues during learning \cite{sutton1998reinforcement,lan2019maxmin}.
It consists of non-terminal states $A$ and $B$, and terminal states
$T1$ and $T2$. An episode begins in state $A$, from which the agent
can choose between two actions: to move left (towards state $B$)
or right (towards state \textbf{$T2$}), both of which return a reward
of $0.$ From state $B$, the agent can choose from $8$ different
actions, all of which take it to state $T1$, but the reward corresponding
to these actions is drawn from a uniform distribution $U(1,-1)$,
with mean $\mu$. That is, $r\sim\mu+U(1,-1)$.

Intuitively, if the the mean reward $\mu$ is positive, the action
of going left from state $A$ is of relatively high value (and thus
benefits from an overestimation bias, as explained in the discussion
following Theorem \ref{thm:theorem1}), and the optimal policy at
state $A$ is to always move left. Similarly, if $\mu$ is negative,
the action of moving left from state $A$ is of relatively low-value
(benefits from underestimation), and the optimal policy is to always
move right, into state $T2$. Corresponding to these settings, we
set $\mu$ to be $+0.1$ and $-0.1$, and test the performance of
several $Q-$learning variants with balanced $Q-$learning.

In order to evaluate the performance of the different agents, we compute
the percentage of instances (over $500$ trials) where left actions
were chosen from state $A$ during training. In the reward setting:
$\mu=+0.1$ (Figure \ref{fig:lineworld_leftactions}(a)), a higher
percentage of left actions is better (ideally, $100\%$). Here, $Q-$learning,
inherently associated with an overestimation bias, performs well,
whereas algorithms characterised by significant underestimation biases
(Double $Q-$learning and Maxmin $Q-$learning ($N=8$) ) perform
relatively poorly. Balanced $Q-$learning, initialized with a balancing factor $\beta=1$, performs on par with $Q-$learning.

For the reward setting: $\mu=-0.1$ (Figure \ref{fig:lineworld_leftactions}(b)),
a lower percentage of left actions is better (ideally, $0\%$). In
this scenario, the overestimation bias associated with $Q-$learning
negatively affects its performance, which is characterised by the
initial peak in the percentage of left actions (shown in Figure \ref{fig:lineworld_leftactions}
(b)), only recovering to close-to-ideal values around $4000$ episodes.
Balanced $Q-$learning also exhibits an initial tendency to overestimate
the value of left actions. This is due to the initialization of balancing factor $\beta$ as $1$. However, it self-corrects to a low percentage
of left actions within just a few hundred episodes (\emph{an order of magnitude
faster} than $Q-$learning). This self-corrective nature allows balanced
$Q-$learning to perform relatively well in both reward settings,
showcasing its ability to remain relatively agnostic to the underlying
bias preferences of the environment. In contrast, the performance
of maxmin $Q-$learning is highly sensitive to the value of $N$ used.
Incorrectly setting $N$ could severely affect the learning performance,
depending on the environment's inherent tendency to prefer overestimation
or underestimation. This is not the case with balanced $Q-$learning.

The ability of balanced $Q-$learning to remain agnostic to different
bias preference scenarios is further demonstrated by creating environments
with different values of the mean reward $\mu$, and tracking the
$\beta'$ value of taking the left action from state $A$ during learning.
For environments with a high value of $\mu$, one would expect optimistic
updates to be preferred, whereas in environments with low $\mu$,
pessimistic updates would be preferred. As depicted in Figure \ref{fig:lineworld_leftactions_betavar}, balanced $Q-$learning maintains high values of $\beta'$ for
environments with high $\mu$, and converges to low values in environments
with low $\mu$, thus allowing it to perform well in a range of environments,
irrespective of their inherent bias preferences.

\textbf{Results on Benchmark Environments:} We first evaluate the
empirical performance of balanced DQN on multiple games from the MinAtar
environment \cite{young2019minatar}, which was designed to decouple
the representational complexity associated with Atari games \cite{bellemare2013arcade,mnih2015human}
from the task of learning useful behaviors. It has been used as a
benchmark environment in recent work exploring methods for controlling
the estimation bias \cite{lan2019maxmin,cini2020deep}.

\begin{figure}
\begin{raggedright}
\begin{minipage}[t]{0.5\textwidth}%
\noindent \begin{center}
\includegraphics[width=1\textwidth]{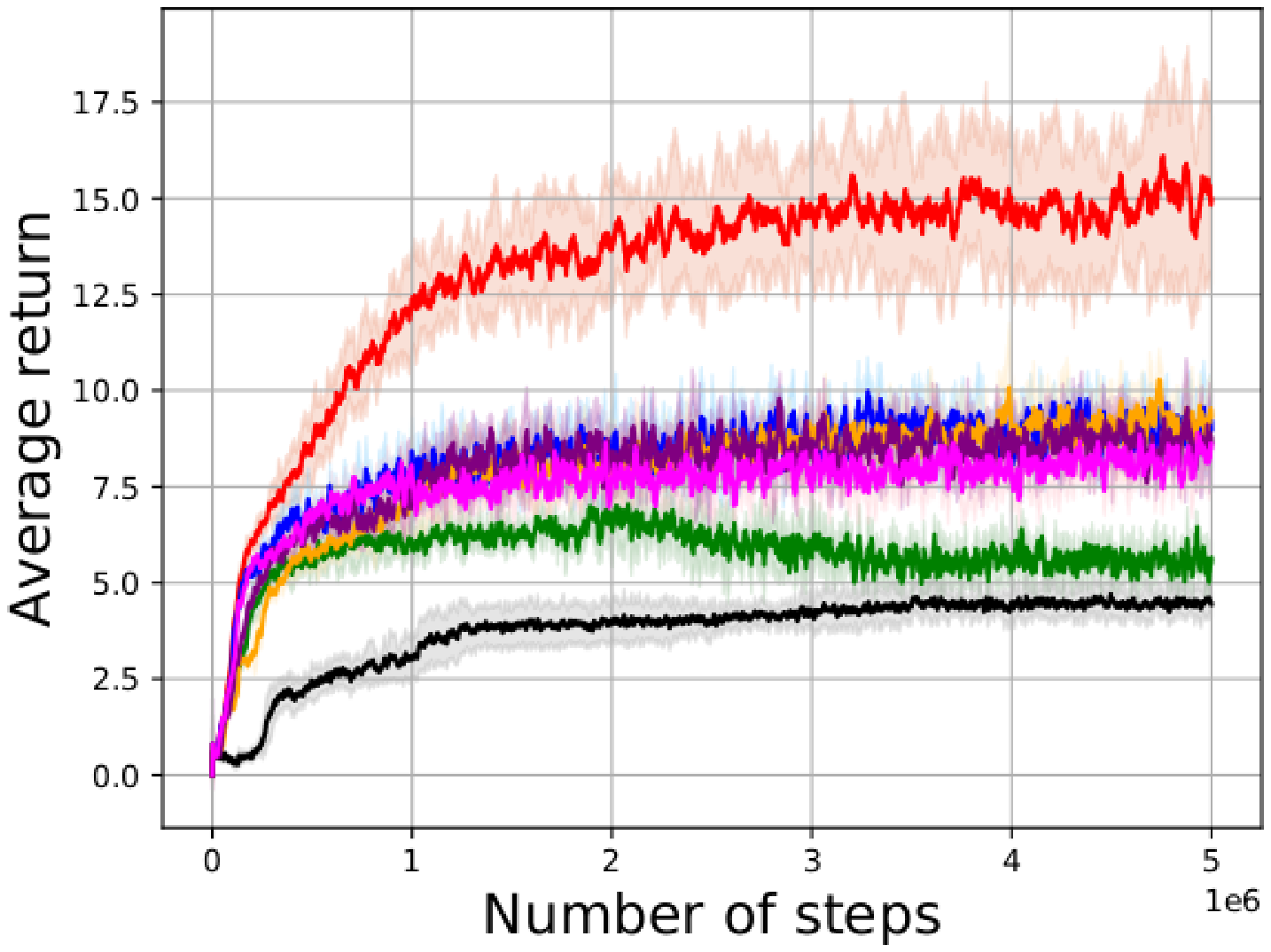} 
\par\end{center}
\begin{center}
(a) Breakout 
\par\end{center}%
\end{minipage}%
\begin{minipage}[t]{0.5\textwidth}%
\noindent \begin{center}
\includegraphics[width=1\textwidth]{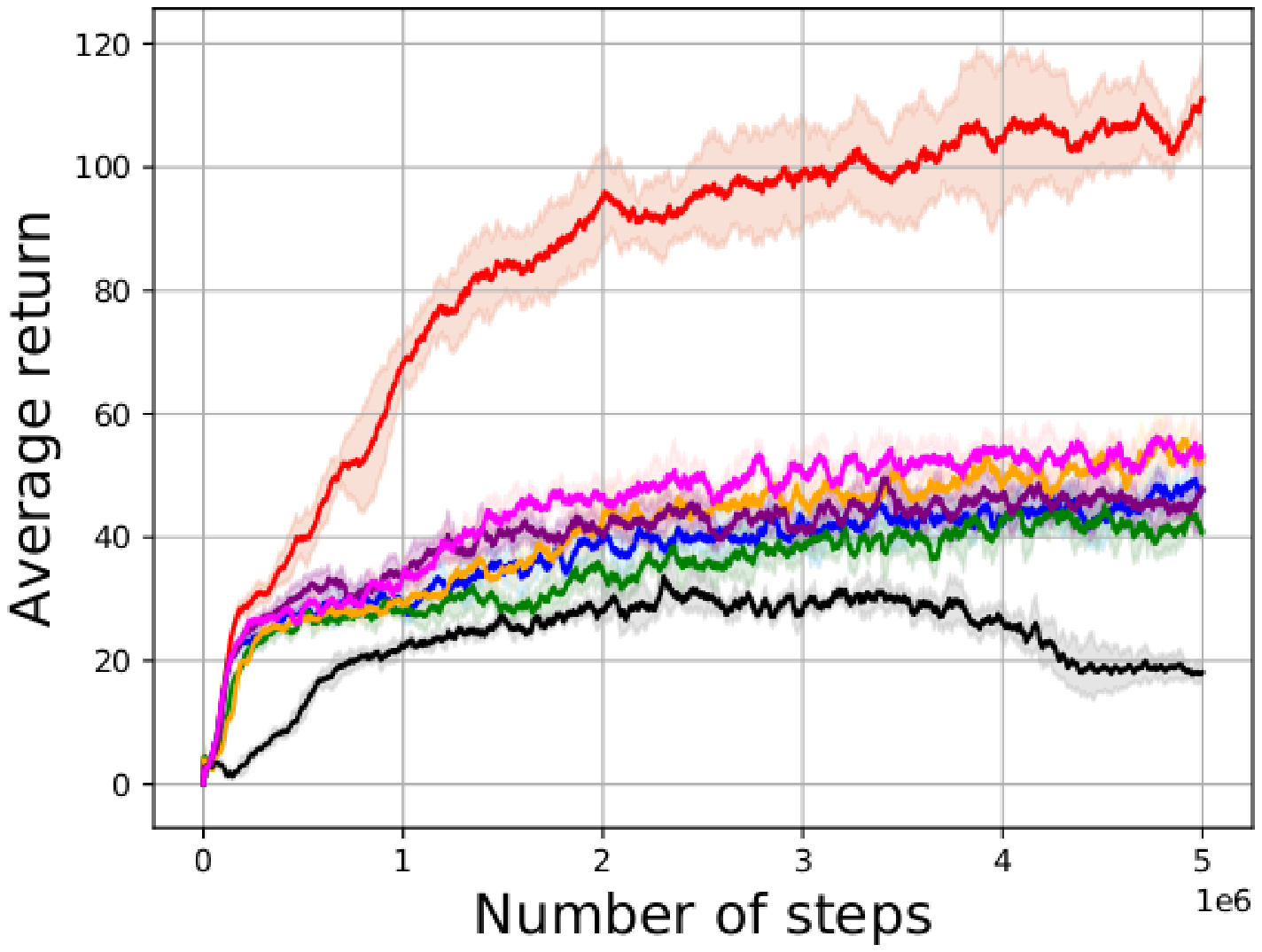} 
\par\end{center}
\begin{center}
(b) Space Invaders 
\par\end{center}%
\end{minipage}%
\\
\begin{minipage}[t]{0.5\textwidth}%
\noindent \begin{center}
\includegraphics[width=1\textwidth]{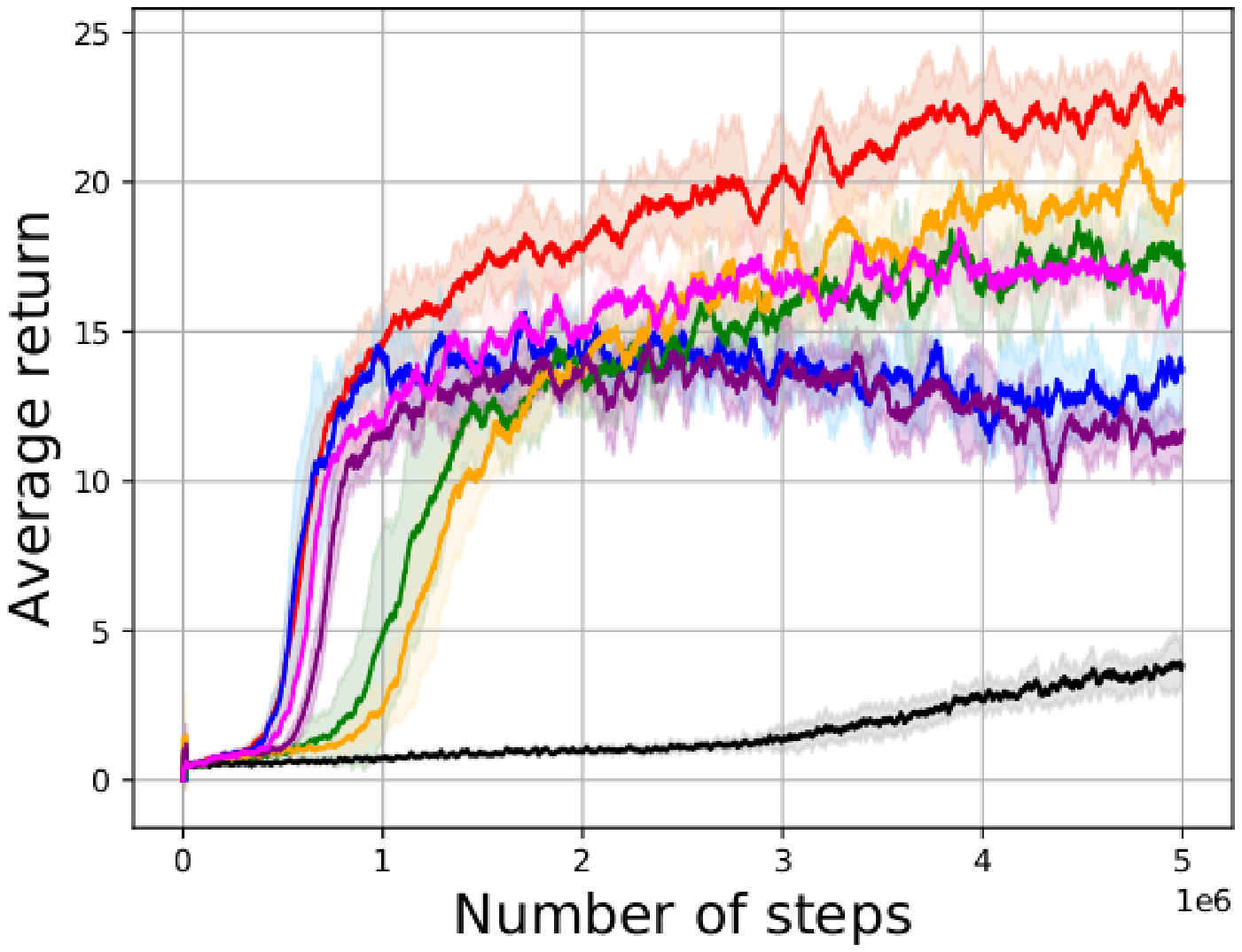} 
\par\end{center}
\begin{center}
(c) Asterix 
\par\end{center}%
\end{minipage}%
\noindent\begin{minipage}[t]{0.25\textwidth}%
\noindent \begin{center}
\includegraphics[width=1\textwidth]{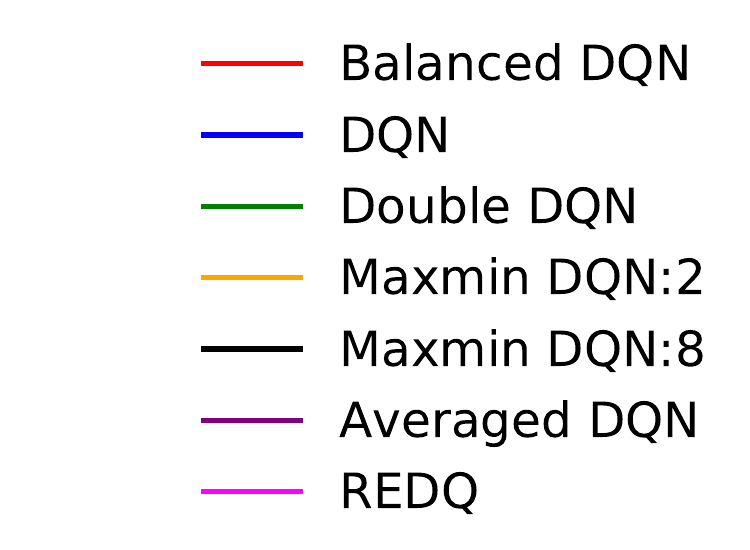}
\par\end{center}%
\end{minipage}
\par\end{raggedright}
\caption{Performance plots on the MinAtar environments (a) Breakout, (b) Space Invaders and (c) Asterix. The results are averaged
over $10$ runs, and the shaded regions represent standard deviation.
\label{fig:Performance-plots}}
\end{figure}

\begin{table}
\centering{}%
\resizebox{\columnwidth}{!}{
\begin{tabular}{c c c c c c c}
\hline 
 & {\scriptsize{}DQN} & {\scriptsize{}Double} & {\scriptsize{}Average } & {\scriptsize{}RedQ} & {\scriptsize{}Maxmin } & {\scriptsize{}Balanced}\tabularnewline
\hline 
\hline 
{\scriptsize{}Breakout} & {\scriptsize{}8.15\textpm 0.61} & {\scriptsize{}5.73\textpm 0.44} & {\scriptsize{}7.86\textpm 0.59} & {\scriptsize{}7.49\textpm 0.55} & {\scriptsize{}7.85\textpm 0.62} & \textbf{\scriptsize{}13.02\textpm 1.43}\tabularnewline
\hline 
{\scriptsize{}Asterix} & {\scriptsize{}12.07\textpm 1.08} & {\scriptsize{}12.18\textpm 1.29} & {\scriptsize{}11.06\textpm 0.90} & {\scriptsize{}13.80\textpm 0.83} & {\scriptsize{}12.80\textpm 1.09} & \textbf{\scriptsize{}17.11\textpm 1.21}\tabularnewline
\hline 
{\scriptsize{}Space Invaders} & {\scriptsize{}37.46\textpm 2.39} & {\scriptsize{}34.29\textpm 2.28} & {\scriptsize{}40.02\textpm 2.44} & {\scriptsize{}44.35\textpm 2.30} & {\scriptsize{}40.57\textpm 1.88} & \textbf{\scriptsize{}84.95\textpm 6.44}\tabularnewline
\hline 
{\scriptsize{}Seaquest} & {\scriptsize{}11.99\textpm 1.87} & {\scriptsize{}7.46\textpm 1.78} & {\scriptsize{}10.56\textpm 2.51} & {\scriptsize{}9.34\textpm 2.03} & {\scriptsize{}6.12\textpm 1.92} & \textbf{\scriptsize{}16.33\textpm 2.43}\tabularnewline
\hline 
{\scriptsize{}Freeway} & {\scriptsize{}45.00\textpm 0.86} & {\scriptsize{}41.99\textpm 2.26} & {\scriptsize{}44.37\textpm 1.01} & {\scriptsize{}45.01\textpm 0.36} & {\scriptsize{}41.34\textpm 1.22} & \textbf{\scriptsize{}45.21\textpm 0.56}\tabularnewline
\hline 
{\scriptsize{}Island Navigation} & {\scriptsize{}30.39\textpm 1.45} & {\scriptsize{}32.89\textpm 1.47} & {\scriptsize{}34.57\textpm 1.05} & {\scriptsize{}24.17\textpm 1.12} & {\scriptsize{}35.36\textpm 0.71} & \textbf{\scriptsize{}36.17\textpm 1.14}\tabularnewline
\hline 
{\scriptsize{}CartPole-v0} & {\scriptsize{}84.94\textpm 8.78} & {\scriptsize{}95.20\textpm 9.08} & {\scriptsize{}110.34\textpm 3.66} & {\scriptsize{}104.66\textpm 4.98} & \textbf{\scriptsize{}111.02\textpm 2.26} & {\scriptsize{}106.48\textpm 3.64}\tabularnewline
\hline 
{\scriptsize{}Tabular Navigation} & {\scriptsize{}0.57\textpm 0.07} & {\scriptsize{}0.47\textpm 0.09} & {\scriptsize{}-} & {\scriptsize{}0.53\textpm 0.10} & {\scriptsize{}0.05\textpm 0.05} & \textbf{\scriptsize{}0.58\textpm 0.08}\tabularnewline
\hline 
\end{tabular}}\caption{Average rewards (mean\textpm standard deviation) across environments
and baselines. Bold represents the highest mean value.\label{tab:table_performances}}
\end{table}
Figure \ref{fig:Performance-plots} depicts the performance in a subset
of the MinAtar environments over $10$ trials. As depicted in Figure
\ref{fig:Performance-plots} and Table \ref{tab:table_performances},
balanced DQN exhibits significant performance improvements in almost
all the MinAtar games, demonstrating its ability to perform consistently
in complex environments. We also consider: a tabular navigation environment
\cite{fernandez_probabilistic_2006}, Cartpole-v0 from OpenAI gym
\cite{brockman2016openai} and the Island Navigation environment
\cite{leike2017ai} to demonstrate the consistency of balanced DQN
across different environments. Table \ref{tab:table_performances}
demonstrates this consistency, which is unlike other approaches such
as maxmin DQN or double DQN, where the performance is highly dependent
on the nature of the environment under consideration. Performance
curves corresponding to the other environments have been included
in Appendix \ref{sec:suppl_perf_plots}. For each experiment, we used
a step size $\eta=0.2$, which was found to be consistently good across
environments (refer Appendix \ref{sec:Ablation}). The performance
of averaged DQN and REDQ learning and maxmin DQN was evaluated using
$N=8$ $Q-$networks for target estimation. We also evaluated maxmin
DQN at $N=2$. Further details on the hyperparameters used have been
specified in Appendix \ref{sec:hyperparams}. 

\subsection{Balanced $Q-$Learning and Exploration:}

From the results thus far, it is evident that Balanced $Q-$learning achieves consistently good performances in multiple environments. As per the intuitions of Lan et al. \cite{lan2019maxmin}, appropriately biasing an agent would cause it to overexplore high-value regions and underexplore low value regions, leading to an overall improvement in the agent's performance. Here, we empirically demonstrate that Balanced $Q-$learning exhibits a similar exploratory nature, which could potentially explain its consistently good performance across environments. We consider the cliff world environment
\cite{sutton1998reinforcement} shown in Figure \ref{fig:cliffworld},  where the agent is tasked with navigating from states $S$ to state
$G$. Transitions into the `Cliff' and goal regions shown in Figure
\ref{fig:cliffworld} are terminal, after which the agent is reset
to the start position. The `Cliff' region is associated with a highly
negative reward of $-100$, whereas all other transitions are associated
with a reward of $-1$. Here, the $\epsilon-$greedy exploration takes
place with $\epsilon=0.1$, ensuring exploration noise throughout
learning. Other hyperparameter settings are set as follows: $\gamma=1$,
$\alpha=0.05$ and for balanced $Q-$learning, $\eta=0.2$.

\begin{figure}
\centering{}\includegraphics[width=0.5\columnwidth]{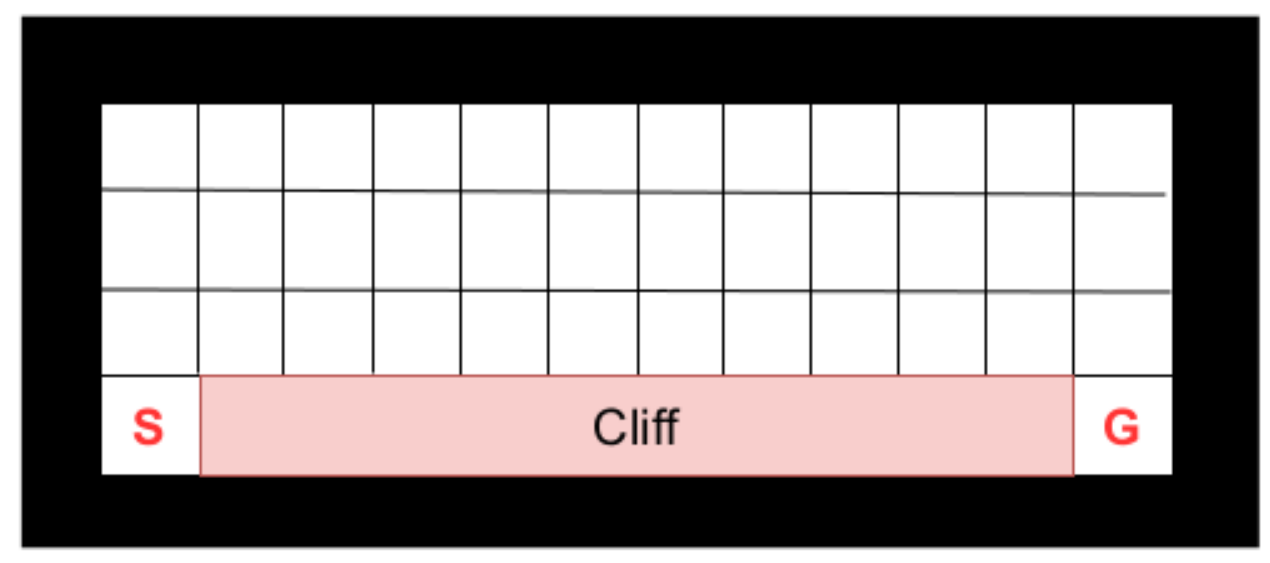}\caption{Cliff World environment. The `Cliff' region is associated with $-100$
reward and all other transitions are associated with $-1$ reward.}
\label{fig:cliffworld} 
\end{figure}

We compare balanced $Q-$learning with standard $Q-$learning by recording
the number of state visits for each state during learning. As depicted
in Figure \ref{fig:visitmap}, with balanced $Q-$learning, the agent
tends to follow more conservative paths, avoiding visiting states
very close to the cliff. This is explained by the fact that balanced
$Q-$learning tends to undervalue low-reward regions, which prevents
excessive exploration into these regions. On the other hand, high-reward regions are overvalued, and the agent encourages exploration into these regions. In contrast to this, standard
$Q-$learning simply follows the optimal path, despite the risk of
falling into the cliff.

\begin{figure}
\centering{}\includegraphics[width=0.5\columnwidth]{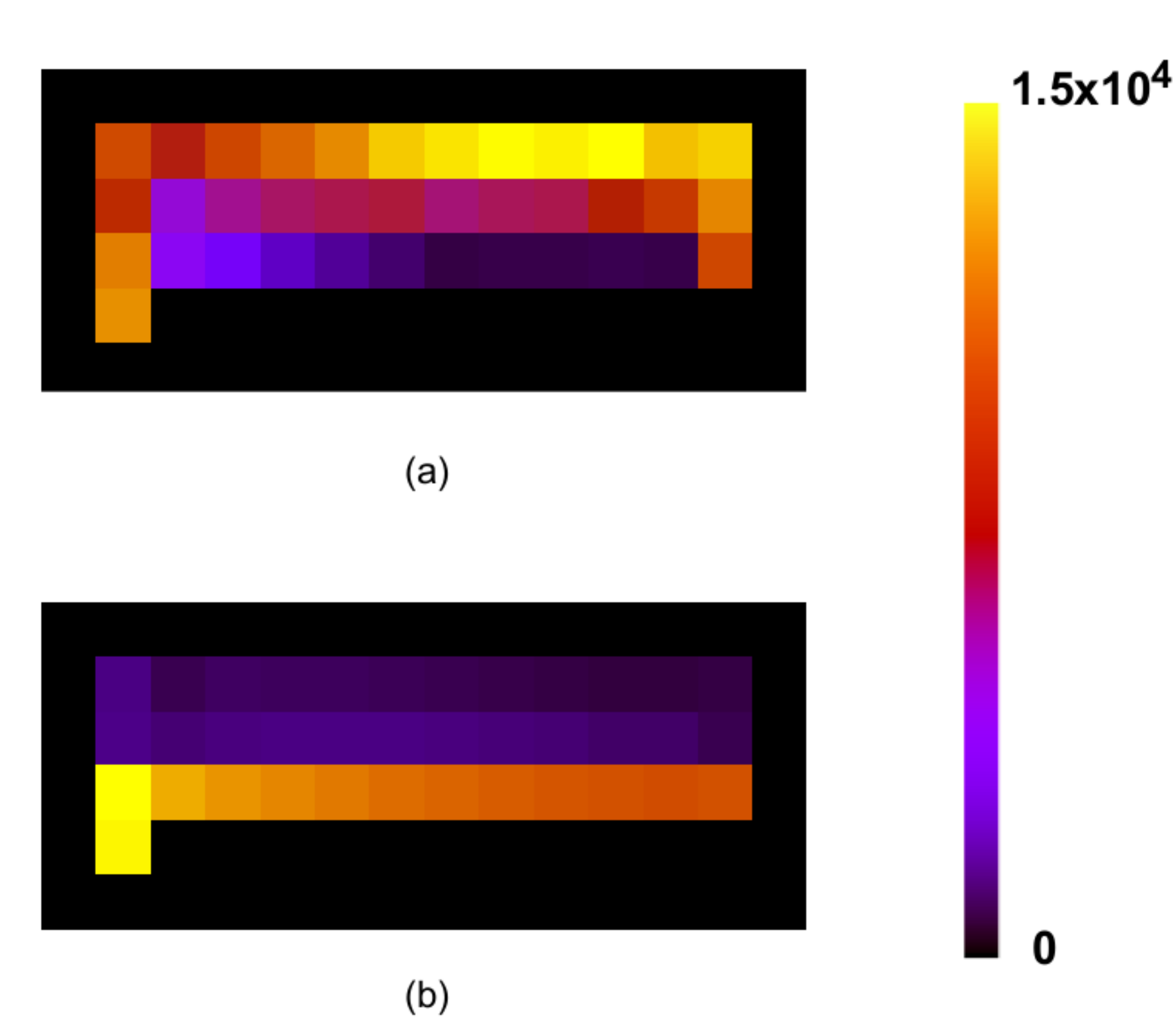}\caption{Visitation heat maps for (a) Balanced $Q-$learning and (b) $Q-$learning
for the cliff world shown in Figure \ref{fig:cliffworld}.}
\label{fig:visitmap} 
\end{figure}

%% file: bg.tex
Thrun and Schwartz \cite{thrun1993issues} first reported the issue
of systematic overestimation in $Q-$learning, which has since had
many proposed solutions. Double $Q-$learning \cite{hasselt2010double}
proposed maintaining two independent $Q-$functions, such that the
expected action value as per one of them is used to choose the best
action for the other. This idea was further extended to the case where
$Q-$functions were approximated using neural networks \cite{van2016deep}.
Although this double network architecture prevents overestimation,
it is accompanied by an underestimation bias, which is also potentially
harmful. Similar to the double architecture, TD3 \cite{fujimoto2018addressing}
also tackled the overestimation problem by maintaining two $Q-$functions,
the minimum of which was used to form the targets in the Bellman error
loss function. 

Recently, maxmin $Q-$learning \cite{lan2019maxmin} was proposed
as a generalized method to control the estimation bias associated
with $Q-$learning. The approach essentially involves maintaining
$N$ $Q-$functions, the minimum values of which are used to construct
the maximization target. Although maxmin $Q-$learning presents the
theoretical possibility of unbiased learning, its performance is highly
sensitive to $N$, which must be chosen beforehand. In addition, the
bias control is limited to discrete changes in $N$, whereas our approach
allows finer control by allowing $\beta$ to assume any real value
in the range $[0,1]$. Kuznetsov el al. \cite{kuznetsov2020controlling}
recently developed an approach for similar finetuned control of overestimation,
but in continuous action settings, using multiple critics. Chen et
al. \cite{chen2021randomized} also proposed an ensemble approach
similar to maxmin $Q-$learning suited to both discrete and continuous
action spaces. However, unlike in maxmin $Q-$learning, they leverage
high update-to-data ratios, and the learning target is determined
on the basis of a subset of all the available networks. Averaged $Q-$learning
\cite{anschel2017averaged} also determines the $Q-$ learning target
using multiple target networks. However, each of the target networks
are used to obtain previous estimates, whose average is used as the
modified target for $Q-$learning. This enabled the reduction of the
variance of the approximation errors of the target, leading to more
stable learning. 

Zhang et al. \cite{zhang2017weighted} proposed an approach to tackle
overestimation, using two estimators: one to estimate the maximum
expected action value, and the other to estimate the action value.
The learning target was expressed as a weighted sum of these estimates.
Although our approach shares similarities in the general objective,
it maintains a single estimator, and the basis for determining the
linear weights is rooted in contextually promoting the right types
of biases. In addition, by weighting the optimistic and pessimistic
targets, our approach can span the full range of biases. A similar
formulation was proposed by Gaskett \cite{gaskett2003reinforcement},
although only fixed weights were considered. Li and Hou\cite{li2019mixing}
also proposed a similar approach using double estimators, and fixed
weights. This limits the applicability of this approach, as it entails
foreknowledge of the appropriate weights for a particular environment.
In contrast to this, our proposed method dynamically adjusts these
weights during learning, raising or lowering it based on an analytically
derived update rule. The objective of doing so is to achieve a consistent
performance across environments, irrespective of the environment's
tendency to inherently suit a particular type of bias.

%% file: conclude.tex
Depending on the environment and the specific region of the state-action
space, both overestimation as well as underestimation bias can potentially
aid learning. Through simple derivations, we showed that overestimation
aids learning in high-value regions, and underestimation is preferable
in low-value regions of the state-action space. We proposed balanced
$Q-$learning, a variant of $Q-$learning where the target is constructed
using a combination of the maximum and minimum action values of the
next state, with the influence of each term being controlled by a
balancing factor. We analytically derived a rule for updating this
factor online, and showed that the resulting algorithm converges in
tabular settings. Through empirical evaluations of the proposed method,
we confirmed its robustness to varying reward structures, as well
as its ability to consistently achieve a good learning performance
in a variety of benchmark environments. 

\looseness=-1Currently, our approach only considers discrete
action settings. Extending the idea of using a balanced target in
continuous action spaces remains to be explored. Although we have
demonstrated consistently good performances across environments, compared
to DQN, we note that our approach requires the storage of two sets
of network parameters (one corresponding to update $n$ and the other
corresponding to $n+1$). This requirement is however relatively more
relaxed in comparison to other ensemble methods, and hence it could
constitute a scalable solution to the problem of achieving risk-aware behaviors by controlling the extent
of overestimation during learning.

%% file: supp.tex
\section{Proofs}

\subsection{Proof of Theorem \ref{thm:theorem1}\label{subsec:Proof-of-Theorem}}
\begin{proof}
Theorem \ref{thm:theorem1} can be proved by induction.

\textit{Base Case}: $m=1$

Substituting $m=1$ in Equation \ref{eq:biasperformance}, we get:
\[
Q^{*}(s,a)-Q_{n+1}(s,a)=(1-\alpha)\left[Q^{*}(s,a)-Q_{n}(s,a)\right]-\alpha t_{n}(s,a)
\]

or

\[
Q_{n+1}(s,a)=Q_{n}(s,a)+\alpha\left[Q^{*}(s,a)+t_{n}(s,a)-Q_{n}(s,a)\right]
\]

This corresponds to the general TD update equation (Equation \ref{eq:gen_TD}),
where the update target is given by $Q^{*}(s,a)+t_{n}(s,a)$. Denoting
$Q_{k}(s,a)$ and $t_{k}(s,a)$ as $Q_{k}$ and $t_{k}$ for brevity,
the above equation becomes:

\[
Q_{n+1}=Q_{n}+\alpha\left[Q^{*}+t_{n}-Q_{n}\right]
\]

\textit{Induction Step:} $m=k+1$

Assuming Equation \ref{eq:biasperformance} is true for $m=k$, we
shall prove that it holds for $m=k+1$

From the general TD equation, we get:

\[
Q_{n+k+1}=Q_{n+k}+\alpha\left[Q^{*}+t_{n+k}-Q_{n+k}\right]
\]

\[
Q^{*}-Q_{n+k+1}=Q^{*}-Q_{n+k}-\alpha\left[Q^{*}+t_{n+k}-Q_{n+k}\right]
\]

\[
Q^{*}-Q_{n+k+1}=(1-\alpha)\left[Q^{*}-Q_{n+k}\right]-\alpha t_{n+k}
\]

Substituting the value for $Q^{*}-Q_{n+k},$we get:

\[
Q^{*}-Q_{n+k+1}=(1-\alpha)\left[(1-\alpha)^{k}\left[Q^{*}-Q_{n}\right]-\alpha\sum_{i=1}^{k}(1-\alpha)^{i-1}t_{n+k-i}\right]-\alpha t_{n+k}
\]

\[
Q^{*}-Q_{n+k+1}=(1-\alpha)^{k+1}\left[Q^{*}-Q_{n}\right]-\alpha\left[(1-\alpha)\sum_{i=1}^{k}(1-\alpha)^{i-1}t_{n+k-i}+t_{n+k}\right]
\]

\[
Q^{*}-Q_{n+k+1}=(1-\alpha)^{k+1}\left[Q^{*}-Q_{n}\right]-\alpha\left[\sum_{i=2}^{k+1}(1-\alpha)^{i-1}t_{n+k-i+1}+t_{n+k}\right]
\]

\[
Q^{*}-Q_{n+k+1}=(1-\alpha)^{k+1}\left[Q^{*}-Q_{n}\right]-\alpha\sum_{i=1}^{k+1}(1-\alpha)^{i-1}t_{n+k+1-i}
\]
\end{proof}
This corresponds to Equation \ref{eq:biasperformance}, with $m=k+1$

\subsection{Proof of Theorem \ref{thm:convergence}\label{subsec:Proof-of-Theorem_convergence}}
\begin{proof}
The convergence of the balanced $Q-$learning is based on the non-expansion
property of the term $\beta'\underset{a'}{max}Q_{n+1}(s',a')+\left[1-\beta'\right]\underset{a'}{min}\thinspace Q_{n+1}(s',a')$
in the update target of balanced $Q-$learning \cite{szepesvari1996generalized}.

In order to prove this, we first begin with Equation \ref{eq:betavalue}
in the tabular case:

\[
\beta'=\beta_{n}+\frac{\eta\delta_{n}}{\gamma[\underset{a'}{max\thinspace}Q_{n+1}(s',a')-\underset{a'}{min\thinspace}Q_{n+1}(s',a')]}
\]

Using Equation \ref{eq:delta_n_beta_0} in the above equation, we
evaluate the expression $\beta'\underset{a'}{max}Q_{n+1}(s',a')+\left[1-\beta'\right]\underset{a'}{min}\thinspace Q_{n+1}(s',a')$
to get:

\begin{multline*}
=\gamma\beta_{n}[\underset{a'}{max\thinspace}Q_{n+1}(s',a')-\underset{a'}{min\thinspace}Q_{n+1}(s',a')]+\frac{\eta}{\gamma}(r(s,a)-Q_{n}(s,a))+\gamma\underset{a'}{min\thinspace}Q_{n+1}(s',a')\\
\\
+\eta[\underset{a'}{\beta_{n}max\thinspace}Q_{n}(s',a')-(1-\beta_{n})\underset{a'}{min\thinspace}Q_{n}(s',a')]\\
\end{multline*}

In the above expression, $0<\gamma\leq1$ and $0\leq\beta_{n}\leq1$.
The max and min operators are non-expansions \cite{szepesvari1996generalized},
and the term $\frac{\eta}{\gamma}(r(s,a)-Q_{n}(s,a))$ is non expansive
as long as $\eta\leq\gamma$. Hence, $\beta'\underset{a'}{max}Q_{n+1}(s',a')+\left[1-\beta'\right]\underset{a'}{min}\thinspace Q_{n+1}(s',a')$
is a non expansion, which satisfies the convergence criterion \cite{szepesvari1996generalized},
under the assumptions $\eta\leq\gamma$, $\underset{n=1}{\overset{\infty}{\sum}}\alpha_{n}(s,a)=\infty$
and $\underset{n=1}{\overset{\infty}{\sum}}\alpha_{n}^{2}(s,a)<\infty$
$\forall(s,a)\in\mathcal{S\times\mathcal{A}}$. 
\end{proof}

\section{Environments and Hyperparameter Settings:\label{sec:hyperparams}}

The tabular navigation environment involves an agent in a discrete
grid world, tasked with navigating to a predetermined location in
the environment. The agent receives a reward of $1$ to reaching the
goal location, and $0$ otherwise. CartPole-v0 is a classical control
task where the agent is tasked with vertically balancing a pole, hinged
on a cart, whose sideways motion can be controlled. Island Navigation
is a continuous state navigation environment designed to evaluate
the safe exploratory behavior of the agent. Here, the goal is to avoid
stepping into `water' locations while navigating to a location in
the environment. 

In the MinAtar environments, each trial was run for $5e6$ steps,
over $10$ trials, with the hyperparameters: batch size=$32$, $\gamma=0.99$,
step size= $2.5e-4$ and replay memory size=$1e5$. The exploration
parameter $\epsilon$ is initially set to $1$, and decayed linearly
over the first $1e5$ steps to a minimum value of $0.1$, after which
it was fixed at this value. The target network was updated every $1000$
steps and the optimizer used was RMSprop with gradient clip $5$.
In all the MinAtar plots, the average return for a step was obtained
by averaging over the previous $100$ episodes. Each MinAtar experiment
was performed on an Nvidia Tesla V100 (32GB) GPU, and on average,
took about 10 hours per trial per environment.

In tabular navigation, the state of the agent is comprised of its
horizontal and vertical positions on the grid, and its goal is to
navigate to a specific location. We solved this environment over $15$
trials ($\alpha=0.05$, $\gamma=0.95$ and $\epsilon$ initially set
to $1$, linearly decaying to $0$ in the final episode), running
it for $1e5$ steps.

For CartPole-v0, we used the following hyperparameter settings: $\alpha=0.001$,
$\gamma=0.95$, steps per episode=$200$, total number of steps=$1e4$.
$\epsilon$ is initially set to $1$, and decays to a minimum value
of $0.01$ as $\epsilon\leftarrow\epsilon\zeta$ after each step,
where $\zeta=0.999$. The function approximator is a $2-$ layered
feedforward neural network with $24$ nodes per hidden layer, with
ReLU activation functions, and trained with a batch size of $32$.

In the Island Navigation environment, we use a $2$ layered feed forward
neural network of $100$ nodes each with ReLU activations. The other
hyperparameters used are: step size $\alpha=0.001$, $\gamma=0.95$,
batch size=$32$, total number of steps=$2e5$. Similar to the cartpole,
environment, the initial value of $\epsilon$ is set to $1$, decaying
to a minimum value of $0.1$ with $\zeta=0.995$. The optimizer used
in Island Navigation and cartpole was Adam \cite{kingma2014adam}. 

In the simple MDP environment in Figure \ref{fig:lineworld}, the
hyperparameters used during learning were: $\eta=0.2$, discount factor
$\gamma=1$, step size $\alpha=0.01$, and $\epsilon$-greedy exploration
with $\epsilon=0.1$. All $Q-$values were initialized with a value
of $0$.

In all environments, REDQ learning, Averaged DQN and Maxmin DQN was
tested with $N=8$ (number of networks). Maxmin $Q-$learning was
also tested with $N=2$. In REDQ learning, we implemented a discrete
action version of the algorithm with the size of the subset of networks
chosen as $M=5$, and the update-to-data ratio $G=1$.

\section{Tabular Implementation:\label{sec:Tabular-Implementation:}}
The tabular implementation of balanced $Q-$learning is shown in Algorithm \ref{alg:Balanced-Q-learning_tabular}.
\begin{algorithm}
\caption{Tabular implementation of Balanced Q-learning\label{alg:Balanced-Q-learning_tabular}}

\begin{algorithmic}[1]

\STATE \textbf{Input: }

\STATE Step sizes \textbf{$\alpha$}, $\eta,$ exploration parameter
$\epsilon$, discount factor $\gamma$, maximum number of steps $N_{max}$

\STATE count $n=1$, $\beta_{1}=1$

\STATE Initialize $Q(s,a)$

\STATE Initialize $Q_{n}(s,a)$ and $Q_{n+1}(s,a)$ as $Q(s,a)$ 

\STATE Get initial state $s$

\STATE \textbf{Output: }Learned value function $Q_{N_{max}}$\textbf{ }

\WHILE {$n\leq N_{max}$} 

\STATE Use $\epsilon-$greedy strategy to choose action $a$; observe
$r,s'$

\STATE Compute TD error $\delta_{n}$ using $Q_{n}$ (Equation \ref{eq:delta_n_beta_0})$:$

\STATE Obtain $\beta'$ as per Equation \ref{eq:betavalue}:

$\beta'=\beta_{n}+\frac{\eta\delta_{n}}{\gamma[\underset{a'}{max\thinspace}Q_{n+1}(s',a')-\underset{a'}{min\thinspace}Q_{n+1}(s',a')]}$

\STATE Clip $\beta'$ to the range $[0,1]$

\STATE Compute target:
\[
Q_{T}=r(s,a)+\gamma[\beta'\underset{a'}{max\thinspace}Q_{n+1}(s',a')+(1-\beta')\underset{a'}{min\thinspace}Q_{n+1}(s',a')]
\]

\STATE Update $\beta_{n+1}$ as $\beta_{n+1}=(n\beta_{n}+\beta')/(n+1)$

\STATE Store $Q_{n}$ as the current estimate: $Q_{n}\leftarrow Q$

\STATE Update $Q$ values: $Q(s,a)\leftarrow Q(s,a)+\alpha[Q_{T}-Q(s,a)]$

\STATE Store $Q_{n+1}$ as the latest estimate: $Q_{n+1}\leftarrow Q$

\STATE Update state: $s\leftarrow s'$

\STATE Update count: $n\leftarrow n+1$

\ENDWHILE

\end{algorithmic} 
\end{algorithm}

\section{Additional Performance Plots:\label{sec:suppl_perf_plots}}

Figure \ref{fig:Performance-plots-1} shows the performance of Balanced
DQN/$Q-$learning in comparison with other baselines, for (a) tabular
navigation (b) CartPole-v0 and (c) Island Navigation. Figure \ref{fig:Performance-plots-1-1}
shows the performance in MinAtar environments (a) Freeway and (b)
Seaquest.

\begin{figure}
\begin{raggedright}
\begin{minipage}[t]{0.5\textwidth}%
\noindent \begin{center}
\includegraphics[width=1\columnwidth]{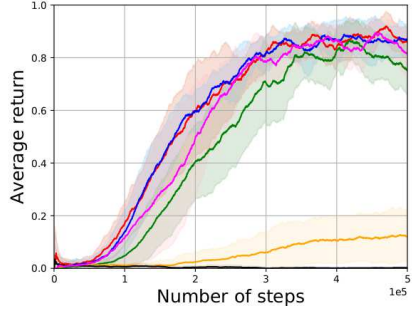} 
\par\end{center}
\begin{center}
(a) Tabular Navigation 
\par\end{center}%
\end{minipage}%
\begin{minipage}[t]{0.5\textwidth}%
\noindent \begin{center}
\includegraphics[width=1\columnwidth]{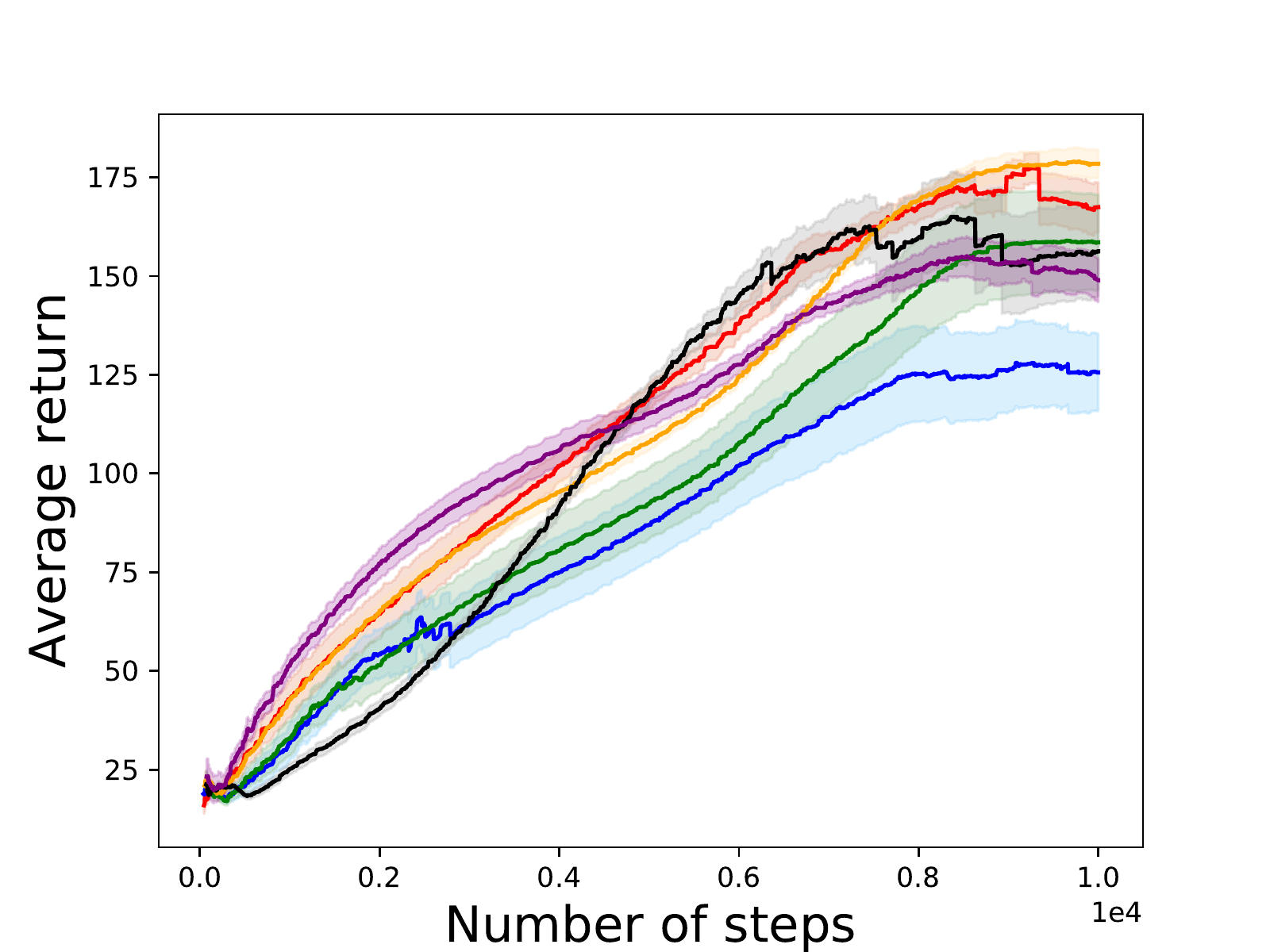} 
\par\end{center}
\begin{center}
(b) CartPole-v0 
\par\end{center}%
\end{minipage}%
\\
\begin{minipage}[t]{0.5\textwidth}%
\noindent \begin{center}
\includegraphics[width=1\columnwidth]{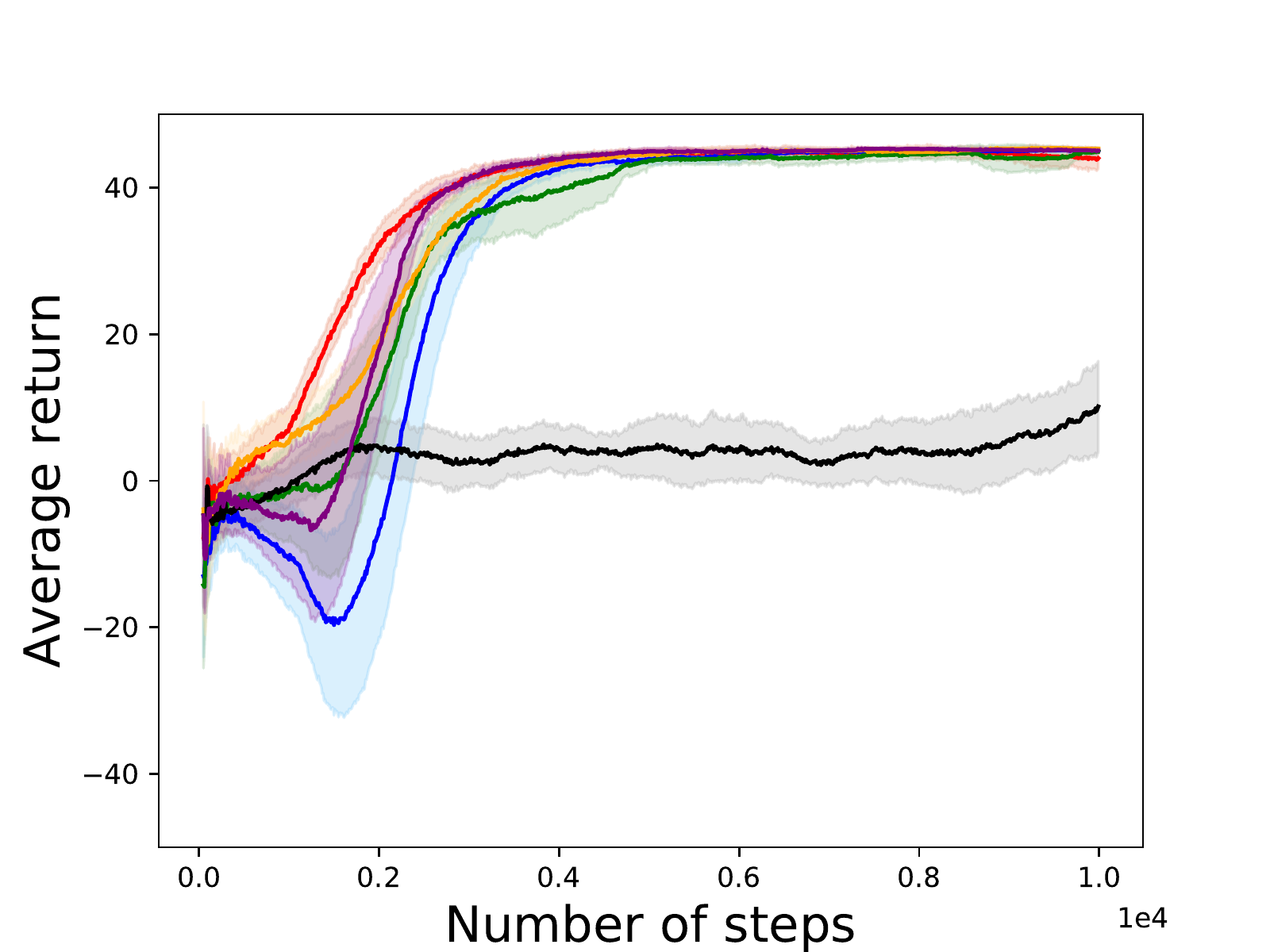} 
\par\end{center}
\begin{center}
(c) Island Navigation 
\par\end{center}%
\end{minipage}%
\noindent\begin{minipage}[t]{0.15\textwidth}%
\noindent \begin{center}
\includegraphics[width=2\columnwidth]{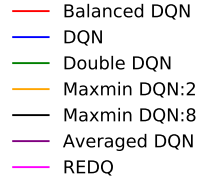}
\par\end{center}%
\end{minipage}
\par\end{raggedright}
\caption{Performance plots on tabular navigation, CartPole-v0 and Island Navigation
environments. The results are averaged over $15$ runs and the shaded
regions represent standard deviation. \label{fig:Performance-plots-1}}
\end{figure}

\begin{figure}
\begin{raggedright}
\begin{minipage}[t]{0.45\textwidth}%
\noindent \begin{center}
\includegraphics[width=1\columnwidth]{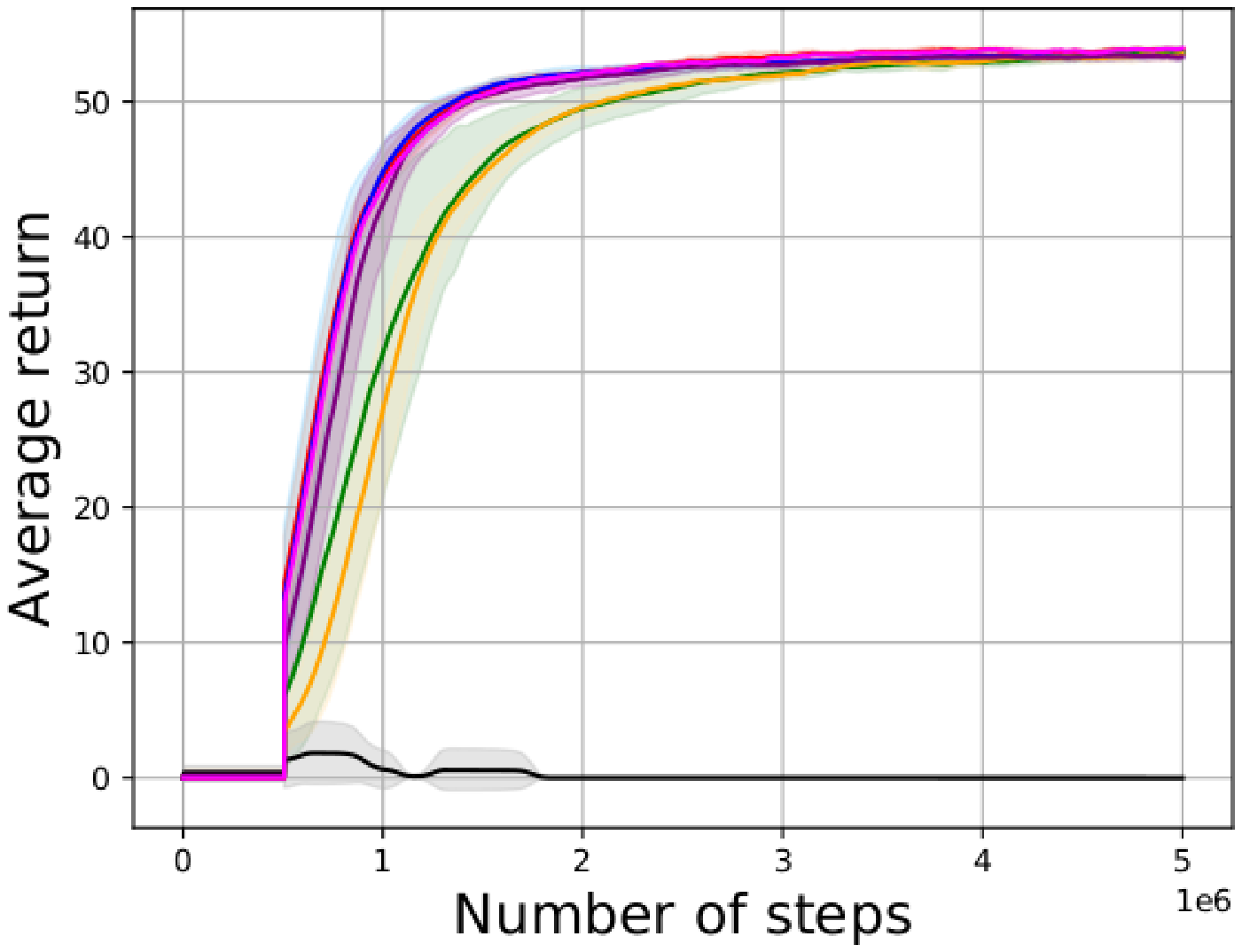} 
\par\end{center}
\begin{center}
(a) Freeway 
\par\end{center}%
\end{minipage}%
\begin{minipage}[t]{0.45\textwidth}%
\noindent \begin{center}
\includegraphics[width=1\columnwidth]{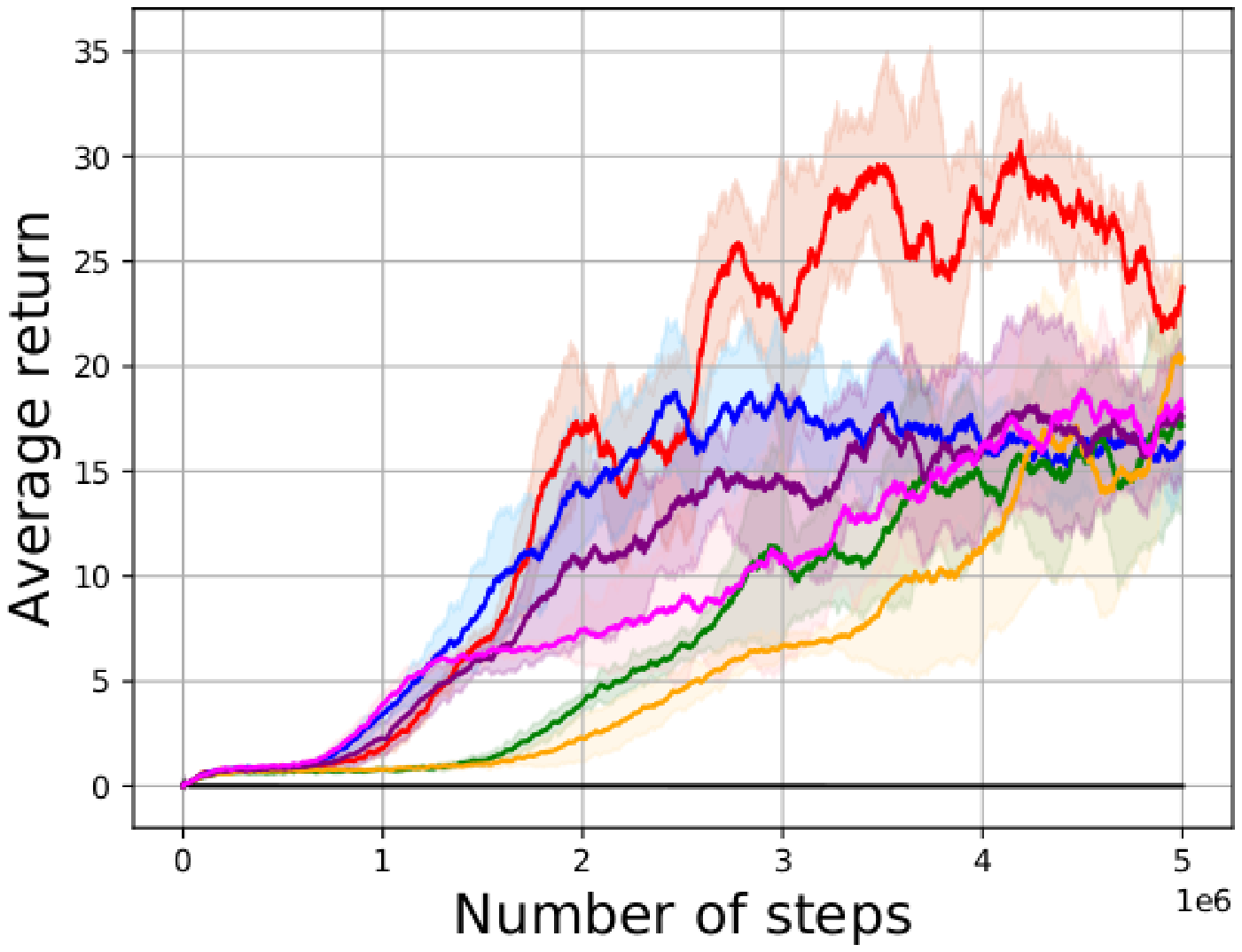} 
\par\end{center}
\begin{center}
(b) Seaquest
\par\end{center}%
\end{minipage}%
\noindent\begin{minipage}[t]{0.1\textwidth}%
\noindent \begin{center}
\includegraphics[width=2\columnwidth]{image_gen/etalegend_DQN}
\par\end{center}%
\end{minipage}
\par\end{raggedright}
\caption{Performance plots on the Seaquest and Freeway MinAtar environments.
The results are averaged over $10$ runs and the shaded regions represent
standard deviation. \label{fig:Performance-plots-1-1}}
\end{figure}

\section{Sensitivity of $\eta$\label{sec:Ablation}}

To show the effect of choosing the hyperparameter $\eta$, we repeat
the experiments with $\eta$ values of $0.2,0.4,0.6,0.8$ and $1$
in each environment. In each case, the ratio of the total average
rewards obtained are plotted relative to the case of $\eta=0$ (which
simply corresponds to the case of DQN). As depicted in Figure \ref{fig:etavar},
the agent produces favorable performances for low values of $\eta$
($\eta>0)$. The performance is unfavorably affected when $\eta$
is set to large values (as is the case when large step sizes are used in general), as this results in large corrections in the
TD error. However, as depicted in Figure \ref{fig:etavar},the algorithm achieves a good performance across a wide range of $\eta$ values.

\begin{figure}
\centering{}\includegraphics[width=0.5\columnwidth]{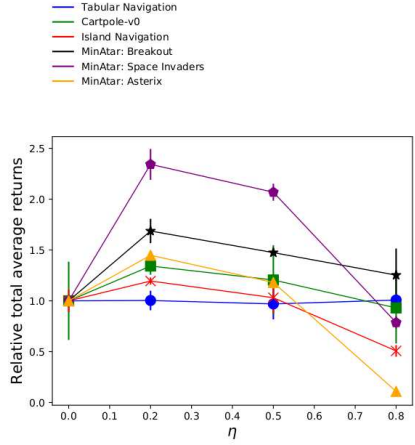}\caption{Performance relative to $\eta=0$ (DQN) for different values of $\eta$
in different environments over $3$ runs. The error bars correspond
to one standard deviation.}
\label{fig:etavar} 
\end{figure}

\section{Fixed $\beta$\label{sec:fixedbeta}}

We also examine the effect of fixed values of $\beta$ in contrast
to determining $\beta$ online using balanced $Q-$learning. As depicted
in Figures \ref{fig:Performance_fixed} and \ref{fig:fixed_simple},
the learning performance is highly sensitive to the specific value
of $\beta$ chosen. However, Balanced DQN automatically updates $\beta$
online to the appropriate value.

\begin{figure}
\begin{raggedright}
\begin{minipage}[t]{0.5\textwidth}%
\noindent \begin{center}
\includegraphics[width=1\columnwidth]{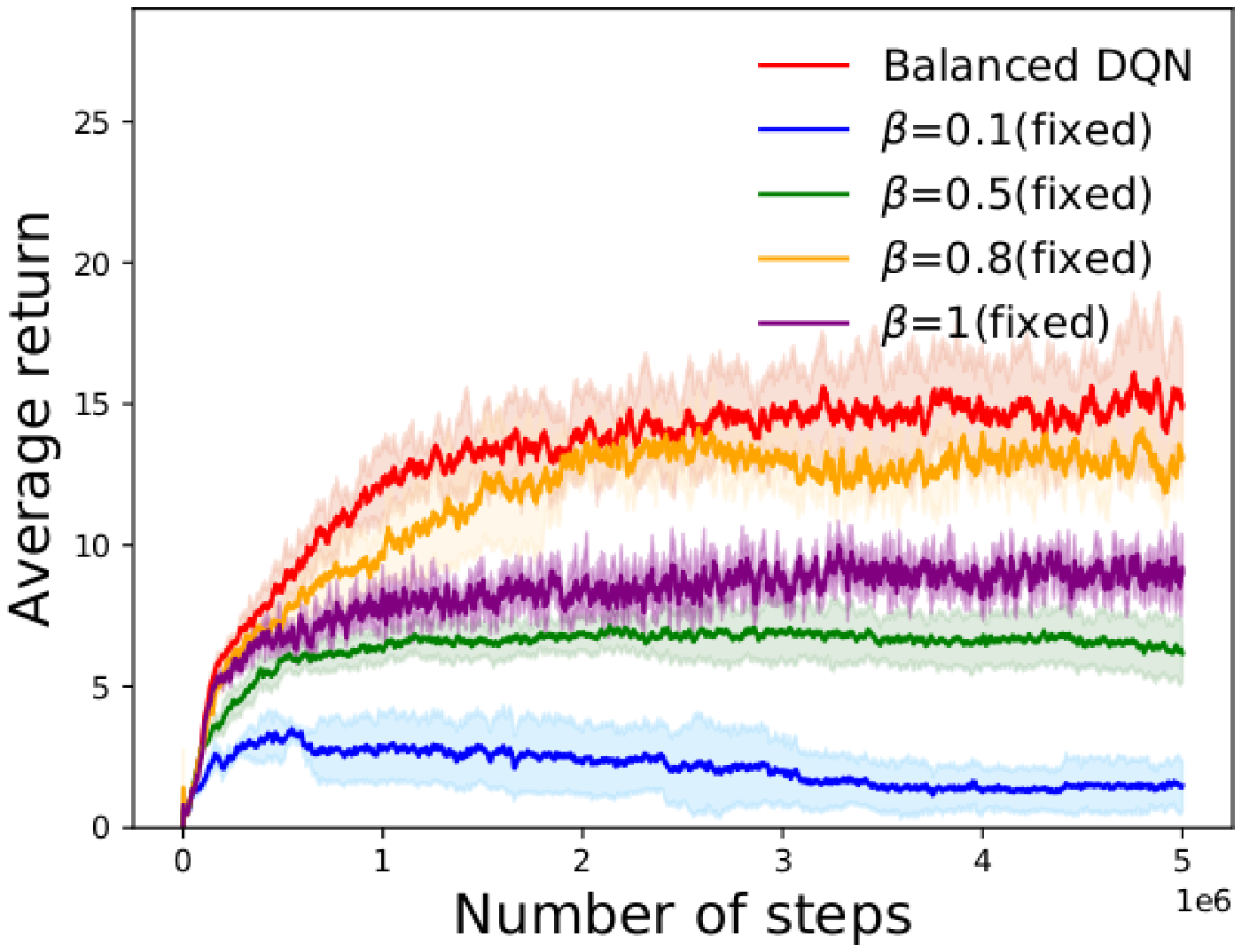} 
\par\end{center}
\begin{center}
(a) Breakout 
\par\end{center}%
\end{minipage}%
\begin{minipage}[t]{0.5\textwidth}%
\noindent \begin{center}
\includegraphics[width=1\columnwidth]{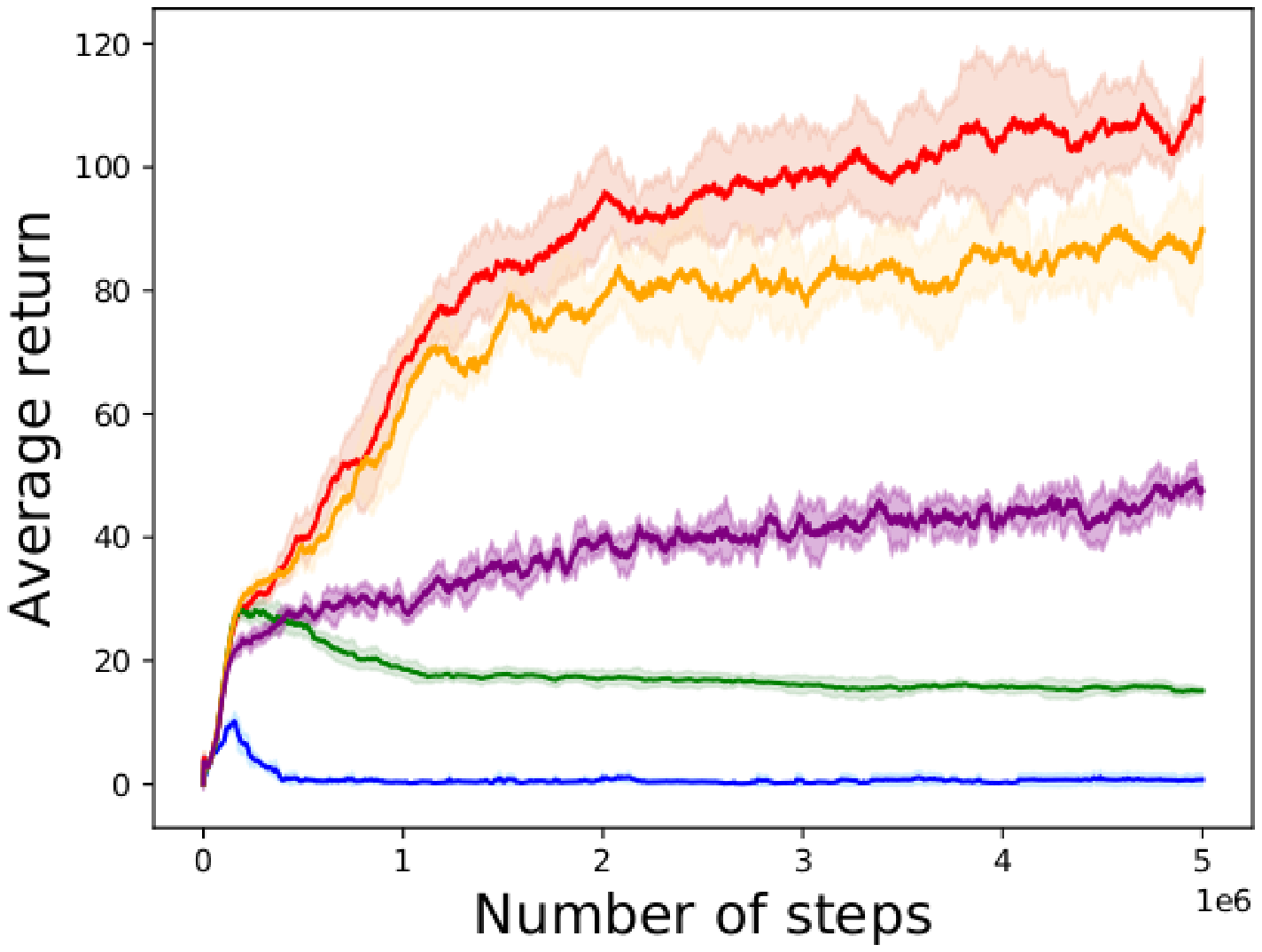} 
\par\end{center}
\begin{center}
(b) Space Invaders 
\par\end{center}%
\end{minipage}%
\\
\begin{minipage}[t]{0.5\textwidth}%
\noindent \begin{center}
\includegraphics[width=1\columnwidth]{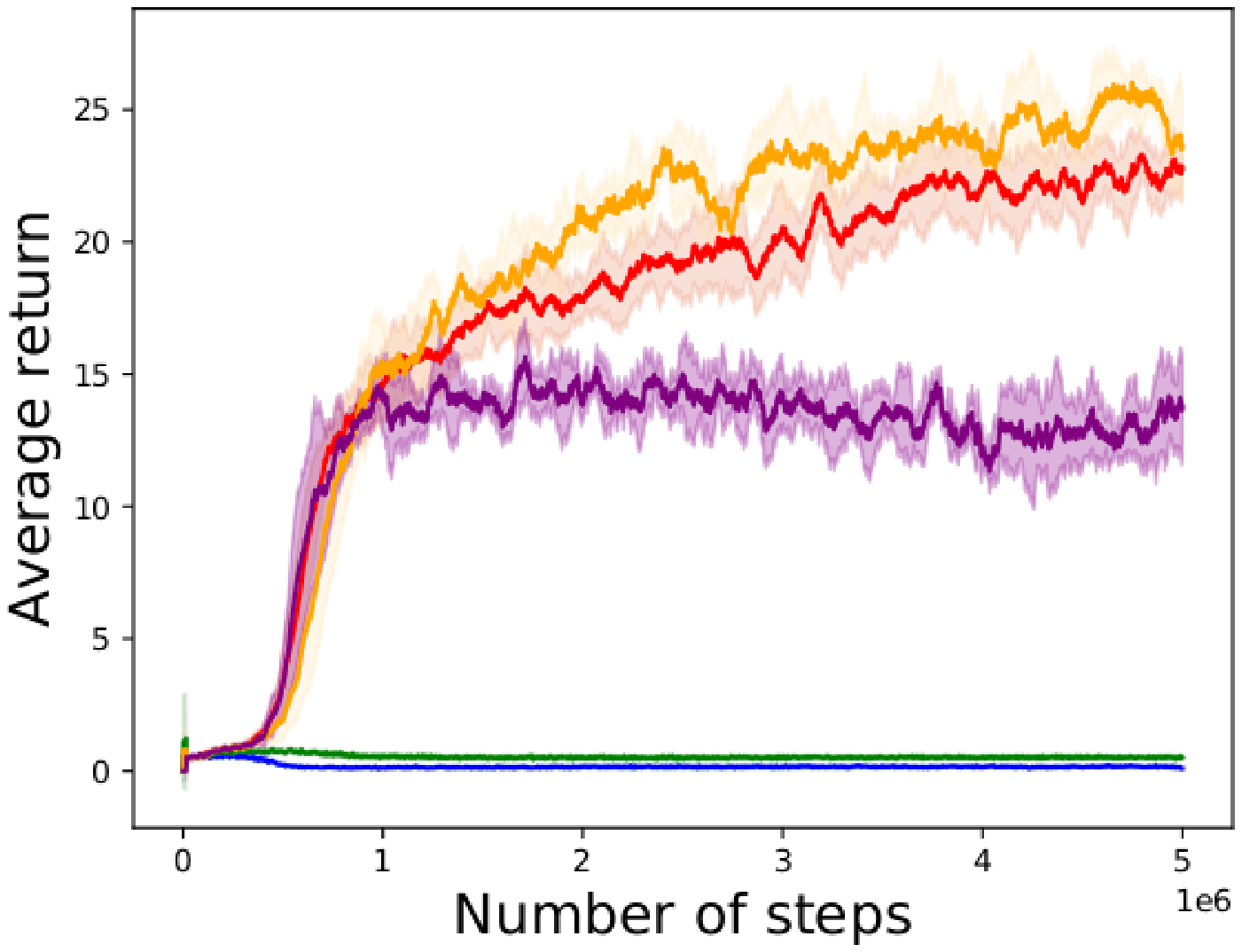} 
\par\end{center}
\begin{center}
(c) Asterix 
\par\end{center}%
\end{minipage}
\par\end{raggedright}
\caption{Performance plots on the MinAtar environments with fixed $\beta$.
The results are averaged over $10$ runs, and the shaded regions represent
one standard deviation. \label{fig:Performance_fixed}}
\end{figure}

\begin{figure}[H]
\begin{raggedright}
\begin{minipage}[t]{0.5\textwidth}%
\noindent \begin{center}
\includegraphics[width=0.5\columnwidth]{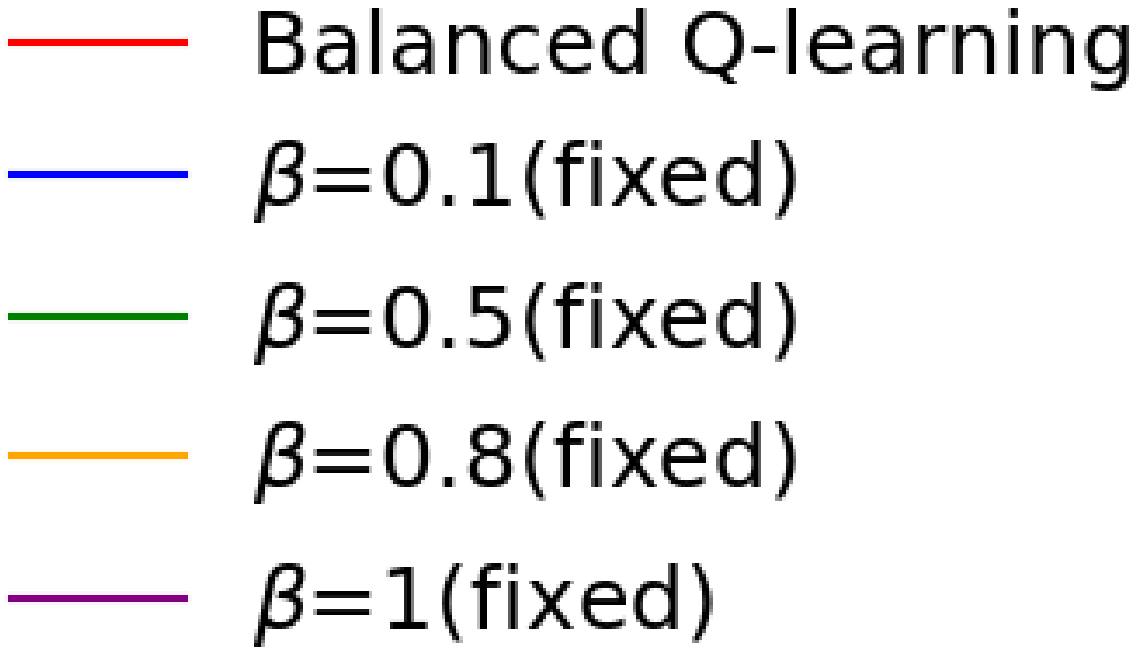} 
\par\end{center}%
\end{minipage}
\par\end{raggedright}
\begin{raggedright}
\begin{minipage}[t]{0.5\textwidth}%
\noindent \begin{center}
\includegraphics[width=1\columnwidth]{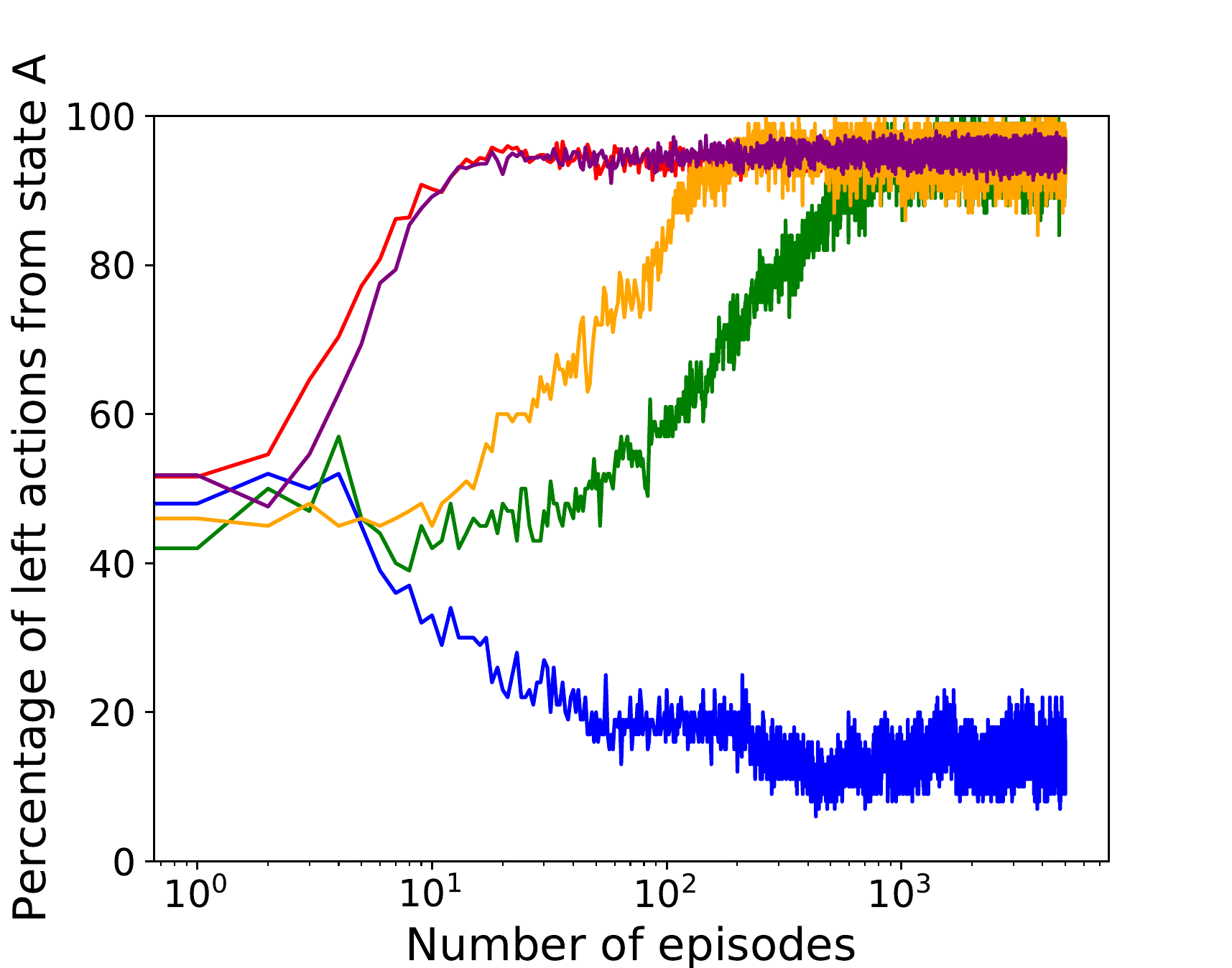} 
\par\end{center}
\begin{center}
(a) Mean reward $\mu=+0.1$. Higher values are better.
\par\end{center}%
\end{minipage}%
\begin{minipage}[t]{0.5\textwidth}%
\noindent \begin{center}
\includegraphics[width=1\columnwidth]{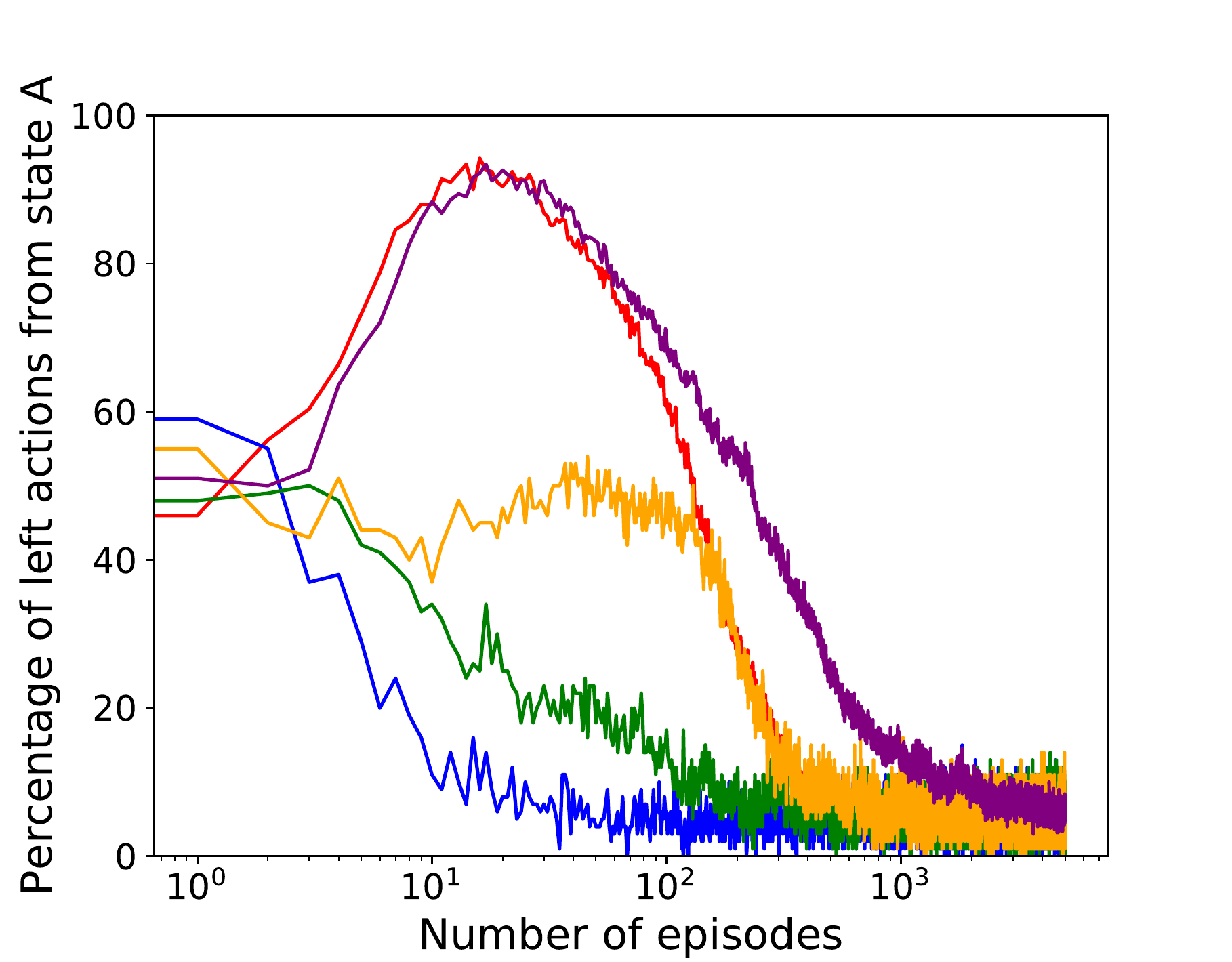} 
\par\end{center}
\begin{center}
(a) Mean reward $\mu=-0.1$. Lower values are better. 
\par\end{center}%
\end{minipage}
\par\end{raggedright}
\caption{Performance in the simple MDP environment (Figure \ref{fig:lineworld})
with fixed $\beta$. The results are averaged over $100$ runs.\label{fig:fixed_simple}}
\end{figure}